\documentclass{article}

\usepackage{microtype}
\usepackage{graphicx}
\usepackage{booktabs}
\usepackage{hyperref}
\usepackage{float}
\graphicspath{{figures/}}

\usepackage[accepted]{icml2026}
\usepackage[T1]{fontenc}
\usepackage{times}
\usepackage{amsmath}
\usepackage{placeins}
\usepackage{needspace}
\usepackage{url,xurl}
\usepackage{amssymb}
\usepackage{mathtools,caption}
\usepackage{amsthm}
\usepackage{xcolor}
\usepackage{tikz}
\usetikzlibrary{arrows.meta, positioning, calc, decorations.markings}
\usetikzlibrary{matrix, positioning}
\usepackage[capitalize,noabbrev]{cleveref}

\theoremstyle{plain}
\newtheorem{theorem}{Theorem}[section]

\theoremstyle{definition}
\newtheorem{definition}[theorem]{Definition}

\theoremstyle{remark}
\newtheorem{remark}[theorem]{Remark}

\definecolor{col1}{RGB}{192,57,43}
\definecolor{col2}{RGB}{41,128,185}
\definecolor{colsum}{RGB}{142,68,173}
\definecolor{darkgray}{RGB}{60,60,60}
\definecolor{black}{RGB}{0,0,0}

\newcommand{\AOEcircle}[8]{%
  \pgfmathsetmacro{\pval}{#4}
  \pgfmathsetmacro{\aval}{#5}
  \pgfmathsetmacro{\bval}{#6}
  \pgfmathsetmacro{\angA}{(\aval/\pval)*360}
  \pgfmathsetmacro{\angB}{(\bval/\pval)*360}
  \pgfmathsetmacro{\startB}{90-\angA}
  \pgfmathsetmacro{\rad}{#3}
  \fill[white] (#1,#2) circle (\rad);
  \pgfmathparse{\angA > 0.01}\ifnum\pgfmathresult=1
    \fill[col1] (#1,#2)
      -- ++({90}:\rad)
      arc[start angle=90, end angle={90-\angA}, radius=\rad]
      -- cycle;
  \fi
  \pgfmathparse{\angB > 0.01}\ifnum\pgfmathresult=1
    \fill[col2] (#1,#2)
      -- ++({\startB}:\rad)
      arc[start angle=\startB, end angle={\startB-\angB}, radius=\rad]
      -- cycle;
  \fi
  \draw[darkgray, line width=0.6pt] (#1,#2) circle (\rad);
  \if1#8
    \pgfmathsetmacro{\sumval}{mod(\aval+\bval,\pval)}
    \foreach \kk in {0, \aval, \sumval}{
      \pgfmathsetmacro{\tickangle}{90 - (\kk/\pval)*360}
      \pgfmathtruncatemacro{\klabelint}{\kk}
      \draw[darkgray, line width=0.8pt]
        (#1,#2) ++(\tickangle:\rad*0.85) -- ++(\tickangle:\rad*0.15);
      \node[font=\tiny, darkgray] at
        ({#1 + (\rad*1.3)*cos(\tickangle)},
         {#2 + (\rad*1.3)*sin(\tickangle)}) {\klabelint};
    }
  \fi
  \fill[darkgray] (#1,#2) circle (0.04);
  \draw[-{Stealth[scale=0.6]}, black, thin]
    (#1,#2) ++({90-\angA*0.5}:\rad*0.38)
    arc[start angle={90-\angA*0.5},
        end angle={90-\angA*0.5-40},
        radius=\rad*0.38];
  \node[font=\small\bfseries, darkgray] at (#1, {#2-\rad-0.32}) {#7};
}

\newcommand{\AOEcircleWrap}[8]{%
  \pgfmathsetmacro{\pval}{#4}
  \pgfmathsetmacro{\aval}{#5}
  \pgfmathsetmacro{\bval}{#6}
  \pgfmathsetmacro{\sumval}{mod(\aval+\bval,\pval)}
  \pgfmathsetmacro{\rad}{#3}
  \pgfmathsetmacro{\angA}{(\aval/\pval)*360}
  \pgfmathsetmacro{\angB}{(\bval/\pval)*360}
  \pgfmathsetmacro{\angSum}{(\sumval/\pval)*360}
  \pgfmathsetmacro{\startB}{90-\angA}
  \fill[white] (#1,#2) circle (\rad);
  \pgfmathparse{\angB > 0.01}\ifnum\pgfmathresult=1
    \fill[col2] (#1,#2)
      -- ++({\startB}:\rad)
      arc[start angle=\startB, end angle={\startB-\angB}, radius=\rad]
      -- cycle;
  \fi
  \pgfmathparse{\angA > 0.01}\ifnum\pgfmathresult=1
    \fill[col1] (#1,#2)
      -- ++({90}:\rad)
      arc[start angle=90, end angle={90-\angA}, radius=\rad]
      -- cycle;
  \fi
  \pgfmathparse{\angSum > 0.01}\ifnum\pgfmathresult=1
    \fill[colsum] (#1,#2)
      -- ++({90}:\rad)
      arc[start angle=90, end angle={90-\angSum}, radius=\rad]
      -- cycle;
  \fi
  \draw[darkgray, line width=0.6pt] (#1,#2) circle (\rad);
  \if1#8
    \foreach \kk in {0, \aval, \sumval}{
      \pgfmathsetmacro{\tickangle}{90 - (\kk/\pval)*360}
      \pgfmathtruncatemacro{\klabelint}{\kk}
      \draw[darkgray, line width=0.8pt]
        (#1,#2) ++(\tickangle:\rad*0.85) -- ++(\tickangle:\rad*0.15);
      \node[font=\tiny, darkgray] at
        ({#1 + (\rad*1.3)*cos(\tickangle)},
         {#2 + (\rad*1.3)*sin(\tickangle)}) {\klabelint};
    }
  \fi
  \fill[darkgray] (#1,#2) circle (0.04);
  \pgfmathsetmacro{\arrowAngle}{90 - \angA - \angB*0.5}
  \draw[-{Stealth[scale=0.6]}, black, thin]
    (#1,#2) ++(\arrowAngle:\rad*0.38)
    arc[start angle=\arrowAngle,
        end angle={\arrowAngle-40},
        radius=\rad*0.38];
  \node[font=\small\bfseries, darkgray] at (#1, {#2-\rad-0.32}) {#7};
}

\newcommand{\AOEcircletwo}[8]{%
  \pgfmathsetmacro{\pval}{#4}
  \pgfmathsetmacro{\aval}{#5}
  \pgfmathsetmacro{\bval}{#6}
  \pgfmathsetmacro{\sumval}{mod(\aval+\bval,\pval)}
  \pgfmathsetmacro{\angSum}{(\sumval/\pval)*360}
  \pgfmathsetmacro{\rad}{#3}
  \fill[white] (#1,#2) circle (\rad);
  \pgfmathparse{\angSum > 0.01}\ifnum\pgfmathresult=1
    \fill[colsum] (#1,#2)
      -- ++({90}:\rad)
      arc[start angle=90, end angle={90-\angSum}, radius=\rad]
      -- cycle;
  \fi
  \draw[darkgray, line width=0.6pt] (#1,#2) circle (\rad);
  \if1#8
    \foreach \kk in {0, \sumval}{
      \pgfmathsetmacro{\tickangle}{90 - (\kk/\pval)*360}
      \pgfmathtruncatemacro{\klabelint}{\kk}
      \draw[darkgray, line width=0.8pt]
        (#1,#2) ++(\tickangle:\rad*0.85) -- ++(\tickangle:\rad*0.15);
      \node[font=\tiny, darkgray] at
        ({#1 + (\rad*1.3)*cos(\tickangle)},
         {#2 + (\rad*1.3)*sin(\tickangle)}) {\klabelint};
    }
  \fi
  \fill[darkgray] (#1,#2) circle (0.04);
  \pgfmathparse{\angSum > 0.01}\ifnum\pgfmathresult=1
    \draw[-{Stealth[scale=0.6]}, black, thin]
      (#1,#2) ++({90-\angSum*0.5}:\rad*0.38)
      arc[start angle={90-\angSum*0.5},
          end angle={90-\angSum*0.5-40},
          radius=\rad*0.38];
  \fi
  \node[font=\small\bfseries, darkgray] at (#1, {#2-\rad-0.38}) {#7};
}

\icmltitlerunning{Prime Fourier Embeddings}

\begin{document}

\twocolumn[
\icmltitle{Prime Fourier Embeddings: A Principled Basis for Modular Arithmetic}

\icmlsetsymbol{equal}{*}

\begin{icmlauthorlist}
\icmlauthor{Hyunsang Hwang}{kor}
\icmlauthor{Suhyun Bae}{kor}
\icmlauthor{Donghun Lee}{kor}
\end{icmlauthorlist}

\icmlaffiliation{kor}{Department of Mathematics, Korea University, Seoul, South Korea}

\icmlcorrespondingauthor{Hyunsang Hwang}{hyunsang@korea.ac.kr}
\icmlcorrespondingauthor{Suhyun Bae}{baeshstar@korea.ac.kr}
\icmlcorrespondingauthor{Donghun Lee}{holy@korea.ac.kr}

\icmlkeywords{Machine Learning, ICML, numerical embeddings,
modular arithmetic, representation theory}

\vskip 0.3in
]
\printAffiliationsAndNotice{%
This work was supported by the National Research Foundation of Korea (NRF) grant funded by the Korea government (MSIT) (RS-2025-24873052).\ }

\begin{abstract}
Numbers have algebraic structure that standard neural embeddings often fail to expose. 
We introduce Prime Fourier Embeddings (PFE), which encode integers as prime-indexed $(\cos, \sin)$ pairs derived from the harmonic analysis of $\mathbb{Q}$, providing a pre-structured representation in which modular arithmetic reduces to selecting the relevant prime channel rather than discovering algebraic structure from scratch. 
We prove that any linear map equivariant with respect to the product group action on PFE must be block-diagonal with one independent block per prime — a consequence of Schur's lemma applied to the resulting character decomposition.
For square-free composite moduli, the Chinese Remainder Theorem predicts which prime channels are task-relevant.
Both predictions are confirmed empirically across a systematic sweep of prime counts, composite moduli, and input ranges: ablation studies confirm the block-diagonal prediction, task-relevant channels cause large accuracy drops when ablated ($0.60$--$0.92$ diagonal drop)
while task-irrelevant channels are effectively inert (off-diagonal drops at or below the statistical noise floor in the majority of configurations).
\end{abstract}

\section{Introduction}

Modular arithmetic has a natural decomposition structure: by the
Chinese Remainder Theorem (CRT), $(a + b) \bmod N$ reduces to
independent additions modulo each prime factor of $N$. A
representation that exposes this structure directly should make
modular arithmetic easy to learn — the network need only identify
which prime channels are relevant to the task, rather than
reconstruct the algebraic structure from scratch. Yet standard
embeddings provide no such structure. Models with learned embeddings
must recover the relevant periodicity through gradient descent, a
task that is both data-hungry and fragile \cite{arith_right, linear}.
The representational bottleneck is not architectural but
informational: no amount of training can recover structure that is
absent from the input.

Prior work has approached this problem from two directions.
Continuous encodings such as xVal \cite{xval} preserve scalar
magnitude but discard the periodic structure that governs modular
arithmetic. Periodic approaches — including trigonometric embeddings
for tabular data \cite{tabular} and Fourier Number Embeddings (FoNE)
\cite{fone} — recover cyclic structure via sinusoidal features, but
the choice of frequency basis matters critically. FoNE uses base-10
periodicities $[\cos(2\pi x/10^k), \sin(2\pi x/10^k)]$, which
couple independent $2$-adic and $5$-adic structures within a single
shared channel since $10 = 2 \times 5$. For modular arithmetic tasks
whose structure is organized by prime factors, base-10 is misaligned:
it does not isolate prime-local structure into independent valuation
channels, leaving their separation as an implicit task for the
network. More closely related, \citet{stevens2024salsafresca}
introduced angular embeddings for integers modulo $q$ in the context
of machine learning attacks on Learning With Errors, mapping each
integer to a single sinusoidal feature at frequency $2\pi/q$. PFE
generalizes this idea: rather than a single frequency tied to the
task modulus, PFE uses a structured family of prime-indexed
frequencies derived from adelic harmonic analysis, making the basis
principled and independent of any particular modulus.

We introduce \textbf{Prime Fourier Embeddings (PFE)}, which encode
integers as prime-indexed $(\cos, \sin)$ pairs at successive digit
positions. Each prime $p$ contributes $D$ independent two-dimensional
blocks, one per digit depth, giving a representation in which the
residue structure of any integer is directly readable from the
corresponding prime channel. PFE extends our prior work
\citep{bae2026numberscarryembeddings} by formalizing the
block-diagonal structure as a provable consequence of Schur's lemma
and providing systematic empirical validation across prime counts,
composite moduli, and embedding variants. The construction is
motivated by the harmonic analysis of $\mathbb{Q}$: by Pontryagin
duality and the factorization of adelic characters
(Appendix~\ref{app:pontryagin}, Theorems~\ref{thm:pontryagin}
and~\ref{thm:adelic}, \cite{ramakrishnan1999fourier}), any character
on $\mathbb{A}_\mathbb{Q}$ factorizes into independent local
components, one per prime place, with no cross-prime coupling. PFE
implements these local components directly, making the prime basis a
principled rather than arbitrary choice. The broader representational
literature confirms this design principle: Poincar\'{e} embeddings
\cite{Poincar} exploit hyperbolic geometry for hierarchical data,
RoPE and RotatE \cite{roformer, rotate} encode relational structure
via rotation groups, and Clifford neural layers \cite{clifford}
preserve physical symmetries — each works by building the relevant
mathematical structure directly into the representation.

With PFE, the CRT prediction becomes concrete and testable: for a
task modulus $N = p_1 \cdots p_k$, the $k$ prime channels
corresponding to the factors of $N$ should carry task-relevant
information, while all other channels should be irrelevant. We
formalize this as the Block-Diagonal Decomposition Theorem
(Theorem~\ref{thm:block_diag}): any linear map equivariant with
respect to the product group action on PFE must be block-diagonal,
with one independent block per prime and zero coupling across primes.
This follows from viewing each PFE block as the real form of a
character of the common additive group $\mathbb{Z}$ and applying
Schur's lemma to the resulting non-isomorphic character components.
Both the structural prediction and the CRT channel-selection
prediction are confirmed empirically across a systematic sweep of
prime counts, composite moduli, and input ranges. The result is an
embedding whose internal structure is theoretically characterized and
empirically verified.

\section{Construction}
\label{sec:construction}

\begin{figure*}[t]
\vskip 0.2in
\begin{center}
\centerline{%
\begin{tikzpicture}[>=Stealth, every node/.style={font=\small}]

\node[font=\normalsize\bfseries, darkgray] at (-5.2, 4.6) {Input embedding};

\node[anchor=east, darkgray] at (-6.5, 3.8) {$a = 9$};
\fill[col1]  (-6.2, 3.55) rectangle (-5.369, 4.05);
\fill[white] (-5.369, 3.55) rectangle (-3.8, 4.05);
\draw[darkgray, thick] (-6.2, 3.55) rectangle (-3.8, 4.05);

\node[anchor=east, darkgray] at (-6.5, 2.9) {$b = 17$};
\fill[col2]  (-6.2, 2.65) rectangle (-4.631, 3.15);
\fill[white] (-4.631, 2.65) rectangle (-3.8, 3.15);
\draw[darkgray, thick] (-6.2, 2.65) rectangle (-3.8, 3.15);

\node[anchor=east, darkgray] at (-6.5, 2.0) {$\text{label} = 3$};
\fill[colsum] (-6.2, 1.75) rectangle (-5.923, 2.25);
\fill[white]  (-5.923, 1.75) rectangle (-3.8, 2.25);
\draw[darkgray, thick] (-6.2, 1.75) rectangle (-3.8, 2.25);

\node[font=\normalsize\bfseries, darkgray] at (-5.2, -2) {Label embedding};
\node[anchor=east, darkgray] at (-6.5, -3.0) {$a{+}b \bmod N$};
\fill[col1]   (-6.2, -3.25) rectangle (-5.369, -2.75);  
\fill[colsum] (-5.369,-3.25) rectangle (-5.092, -2.75); 
\fill[col2]   (-5.092,-3.25) rectangle (-3.800, -2.75); 
\draw[darkgray, thick] (-6.2, -3.25) rectangle (-3.8, -2.75);

\draw[col1, line width=0.8pt]
  (-6.2, -3.45) -- (-6.2, -3.55) -- (-5.092, -3.55) -- (-5.092, -3.45);
\node[font=\footnotesize, col1] at ({(-6.2 + -5.092)/2}, -3.70) {$a$};

\draw[col2, line width=0.8pt]
  (-5.369, -2.55) -- (-5.369, -2.45) -- (-3.800, -2.45) -- (-3.800, -2.55);
\node[font=\footnotesize, col2] at ({(-5.369 + -3.800)/2}, -2.30) {$b$};
\draw[->, darkgray, line width=1pt]
  (-3.6, 3.35) -- (-0.5, 3.35)
  node[midway, above, font=\small\itshape] {PFE encode};

\AOEcircleWrap{1}{4}{0.7}{23}{9}{17}{$p=23$}{1}
\draw[green!55!black, line width=2.0pt, rounded corners=3pt]
  (-0.1, 2.8) rectangle (2.1, 5.1);
\node[font=\footnotesize\bfseries, green!55!black, anchor=south west]
  at (-0.1, 5.1) {active $(p=N)$};

\AOEcircle{3}{2}{0.7}{29}{9}{17}{$p=29$}{1}

\AOEcircle{5}{0}{0.7}{31}{9}{17}{$p=31$}{1}

\AOEcircle{7}{-2}{0.7}{37}{9}{17}{$p=37$}{1}

\draw[darkgray!50, dashed, line width=0.8pt, rounded corners=4pt]
  (-0.2, 5.6) rectangle (8.3, -3.5);

\draw[->, darkgray, line width=1pt]
  (-0.5,-3.0) -- (-3.6,-3.0)
  node[midway, below, font=\small\itshape] {PFE decode};
\fill[col1]   (0.5,-4.5) rectangle (1.0,-4.2);
\node[anchor=west, font=\footnotesize, darkgray] at (1.1,-4.35)
  {$a \bmod p$};

\fill[col2]   (3.0,-4.5) rectangle (3.5,-4.2);
\node[anchor=west, font=\footnotesize, darkgray] at (3.6,-4.35)
  {$b \bmod p$};

\fill[colsum] (5.5,-4.5) rectangle (6.0,-4.2);
\node[anchor=west, font=\footnotesize, darkgray] at (6.1,-4.35)
  {excess arc};

\end{tikzpicture}}
\caption{PFE encoding for $(9 + 17) \bmod 23 = 3$. The active 
prime $p=23$ encodes the wrap-around addition where the purple region, in particular, marks the overlap between the red and blue 
arcs — geometrically, the portion of the circle claimed by 
both $a$ and $b$ when their sum exceeds the modulus ($= p$). Its angular size 
is $(a + b - p)/p$, equal to $(a+b) \bmod p$ normalized by 
$p$, which is the label. Inactive primes 
($p=29,31,37$) carry no modulus-specific information.}
\label{fig:aoe_single}
\end{center}
\vskip -0.2in
\end{figure*}
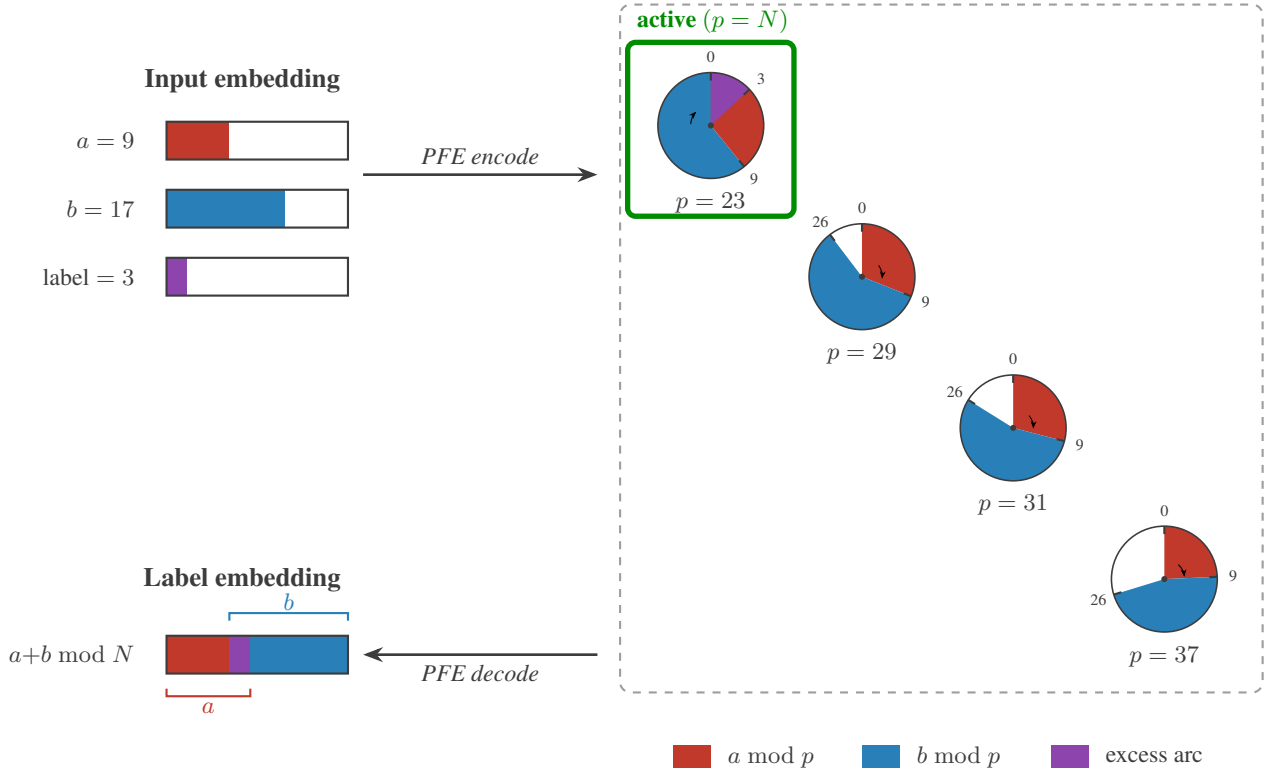

\begin{figure*}[t]
\vskip 0.2in
\begin{center}
\centerline{%
\begin{tikzpicture}[>=Stealth, every node/.style={font=\small}]
 
 
\node[font=\normalsize\bfseries, darkgray] at (-5.2, 4.6) {Input embedding};
 
\node[anchor=east, darkgray] at (-6.5, 3.8) {$a = 13$};
\fill[col1]  (-6.2, 3.55) rectangle (-4.714, 4.05);
\fill[white] (-4.714, 3.55) rectangle (-3.8, 4.05);
\draw[darkgray, thick] (-6.2, 3.55) rectangle (-3.8, 4.05);

\node[anchor=east, darkgray] at (-6.5, 2.9) {$b = 15$};
\fill[col2]  (-6.2, 2.65) rectangle (-4.486, 3.15);
\fill[white] (-4.486, 2.65) rectangle (-3.8, 3.15);
\draw[darkgray, thick] (-6.2, 2.65) rectangle (-3.8, 3.15);

\node[anchor=east, darkgray] at (-6.5, 2.0) {$\text{label} = 7$};
\fill[colsum] (-6.2, 1.75) rectangle (-5.400, 2.25);
\fill[white]  (-5.400, 1.75) rectangle (-3.8, 2.25);
\draw[darkgray, thick] (-6.2, 1.75) rectangle (-3.8, 2.25);
 
\node[font=\normalsize\bfseries, darkgray] at (-5.2, -2.5) {Label embedding};

\node[anchor=east, darkgray] at (-6.5, -3.5) {$a{+}b \bmod N$};
\fill[col2]   (-6.2, -3.75) rectangle (-3.800, -3.25);   
\fill[colsum] (-6.2, -3.75) rectangle (-5.400, -3.25);   
\fill[col1]   (-5.400, -3.75) rectangle (-4.714, -3.25); 
\draw[darkgray, thick] (-6.2, -3.75) rectangle (-3.8, -3.25);

\draw[col1, line width=0.8pt]
  (-6.2, -3.95) -- (-6.2, -4.05) -- (-4.714, -4.05) -- (-4.714, -3.95);
\node[font=\footnotesize, col1] at ({(-6.2 + -4.714)/2}, -4.20) {$a$};

\draw[col2, line width=0.8pt]
  (-4.714, -3.05) -- (-4.714, -2.95) -- (-3.800, -2.95) -- (-3.800, -3.05);
\draw[col2, line width=0.8pt]
  (-6.2, -3.05) -- (-6.2, -2.95) -- (-5.400, -2.95) -- (-5.400, -3.05);
\node[font=\footnotesize, col2] at ({(-4.714 + -3.800)/2}, -2.80) {$b$};

\draw[->, darkgray, line width=1pt]
  (-3.6, 3.35) -- (-0.5, 3.35)
  node[midway, above, font=\small\itshape] {PFE encode};
 
 
\AOEcircletwo{1}{4}{0.7}{3}{1}{0}{$p=3$}{1}

\draw[green!55!black, line width=2.0pt, rounded corners=3pt]
  (-0.1, 2.7) rectangle (2.1, 5.15);
\node[font=\footnotesize\bfseries, green!55!black, anchor=south west]
  at (-0.1, 5.15) {active ($p \mid N$)};
 
\AOEcircletwo{3}{2}{0.7}{5}{3}{0}{$p=5$}{1}
 
\AOEcircletwo{5}{0}{0.7}{7}{0}{0}{$p=7$}{1}
\draw[green!55!black, line width=2.0pt, rounded corners=3pt]
  (3.9, -1.3) rectangle (6.1, 1.15);
\node[font=\footnotesize\bfseries, green!55!black, anchor=south west]
  at (3.9, 1.15) {active ($p \mid N$)};

\AOEcircletwo{7}{-2}{0.7}{11}{6}{0}{$p=11$}{1}


\node[anchor=east, darkgray, font=\footnotesize] at (-7.0, -0.25) {$3{+}3{+}1$};
\fill[colsum]   (-7, 0) rectangle (-6.4, -0.5);  
\draw[darkgray, thick] (-7, 0) rectangle (-6.4, -0.5);
\fill[colsum]   (-6.4, 0) rectangle (-5.8, -0.5);  
\draw[darkgray, thick] (-6.4, 0) rectangle (-5.8, -0.5);
\fill[colsum]   (-5.8, 0) rectangle (-5.6, -0.5);  
\draw[darkgray, thick] (-5.8, 0) rectangle (-5.2, -0.5);
\draw[darkgray, thick] (-5.2, 0) rectangle (-4.6, -0.5);
\draw[darkgray, thick] (-4.6, 0) rectangle (-4.0, -0.5);
\draw[darkgray, thick] (-4.0, 0) rectangle (-3.4, -0.5);
\draw[darkgray, thick] (-3.4, 0) rectangle (-2.8, -0.5);
\node[anchor=east, darkgray, font=\footnotesize] at (-7.0, -1.15) {$7+0$};
\fill[colsum]   (-7, -0.9) rectangle (-5.6, -1.4);  
\draw[darkgray, thick] (-7, -0.9) rectangle (-5.6, -1.4);
\draw[darkgray, thick] (-5.6, -0.9) rectangle (-4.2, -1.4);
\draw[darkgray, thick] (-4.2, -0.9) rectangle (-2.8, -1.4);

\draw[->, darkgray, line width=0.8pt]
  (-5.0, -1.65) -- (-5.0, -2.05);
 
\draw[darkgray!50, dashed, line width=0.8pt, rounded corners=4pt]
  (-0.3, 5.7) rectangle (8, -3.3);
 
\draw[->, darkgray, line width=1pt]
  (-0.5,-0.7) -- (-2.5,-0.7)
  node[midway, below, font=\small\itshape] {PFE decode};
 
\fill[colsum]   (0.5,-4.0) rectangle (1.0,-3.7);
\node[anchor=west, font=\footnotesize, darkgray] at (1.1,-3.85)
  {summed arc};
\end{tikzpicture}}
\caption{PFE encoding for $(13 + 15) \bmod 21 = 7$. Since 
$21 = 3 \times 7$, the prime channels $p=3$ and $p=7$ are load-bearing, carrying residues $(13+15) \bmod 3 = 1$ and 
$(13+15) \bmod 7 = 0$ respectively. The intermediate bars 
show CRT reconstruction: the unique $c \in \mathbb{Z}/21\mathbb{Z}$ 
satisfying $c \equiv 1 \pmod{3}$ and $c \equiv 0 \pmod{7}$ is 
$c = 7$, recovered as $3+3+1 \pmod{21} = 7+0 \pmod{21} = 7$.}
\label{fig:aoe_crt}
\end{center}
\vskip -0.2in
\end{figure*}
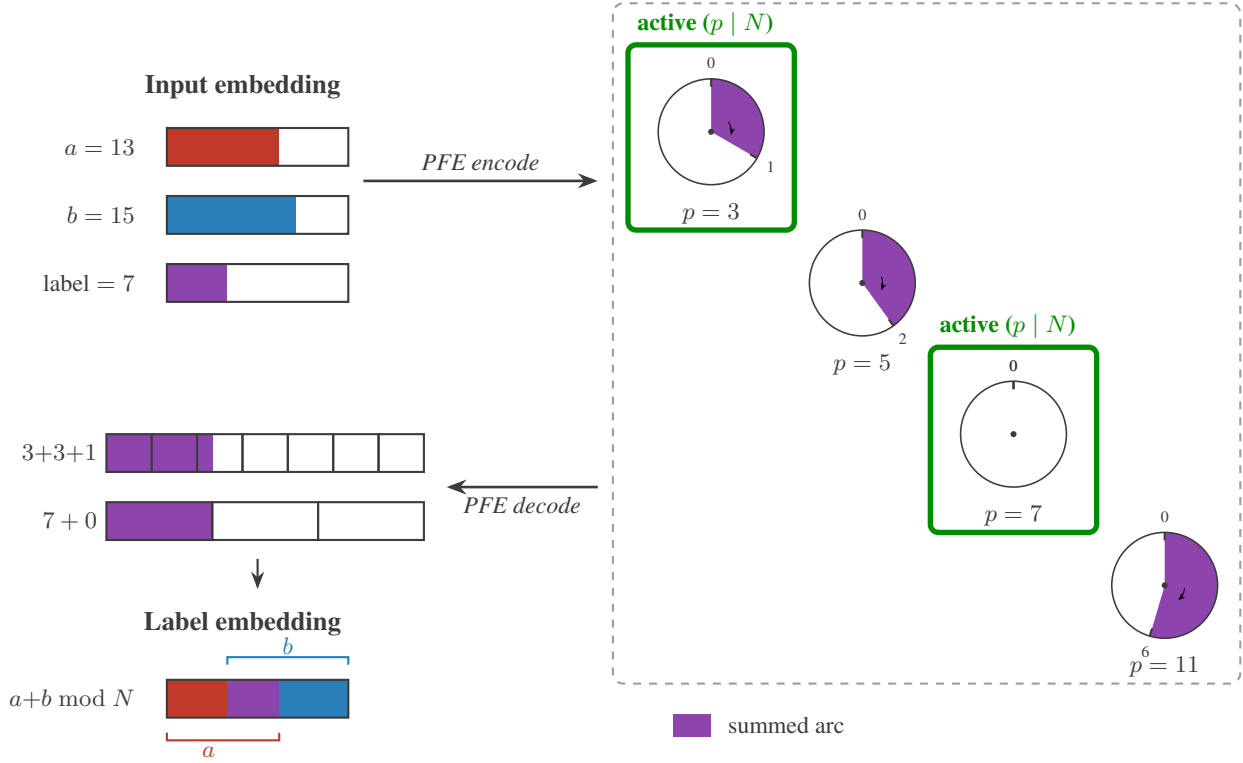

The {Prime Fourier Embeddings (PFE)} for a prime $p$ at 
depth $d\ (0\le d\le D-1)$ maps an integer $a$ to the pair
\[
    \mathrm{PFE}_{p,d}(a) = \left[\cos\!\left(\tfrac{2\pi a}{p^{d+1}}
    \right),\ \sin\!\left(\tfrac{2\pi a}{p^{d+1}}\right)\right]^{\!\top}
    \in \mathbb{R}^2.
\]
This is the real decomposition of the character $\chi_1$ of 
$\mathbb{Z}/p^{d+1}\mathbb{Z}$, recording its real and imaginary 
parts:
\[
    \chi_1(a) = e^{2\pi i a / p^{d+1}} = \cos\!\left(\tfrac{2\pi a}
    {p^{d+1}}\right) + i\sin\!\left(\tfrac{2\pi a}{p^{d+1}}\right).
\]
Since $\chi_1$ is a group homomorphism,
\[
    \mathrm{PFE}_{p,d}(a + b \bmod p^{d+1}) = 
    R\!\left(\tfrac{2\pi b}{p^{d+1}}\right) \cdot \mathrm{PFE}_{p,d}(a),
\]
where $R(\theta)$ denotes rotation by $\theta$; here and throughout, $\mathrm{PFE}_{p,d}(a)$ is a column vector, so rotations --- and the equivariant maps of Section~\ref{sec:theory} --- act by left multiplication. Addition modulo 
$p^{d+1}$ therefore acts on $\mathrm{PFE}_{p,d}$ by rotation, 
making the embedding equivariant with respect to the action of 
$\mathbb{Z}/p^{d+1}\mathbb{Z}$. The $(\cos, \sin)$ pairs are not a 
heuristic choice but a 
principled one: they are a real form of the local $p$-adic 
characters arising from the adelic Fourier decomposition of 
$\mathbb{Q}$ \cite{ramakrishnan1999fourier}.

The truncation to depth $D$ has an algebraic interpretation.
For every integer $a$, $\mathrm{PFE}_{p,d}(a)$ equals the $p$-adic 
character value $\psi_p(a/p^{d+1})$, and this value depends only on 
the residue $a \bmod p^{d+1}$ — the image of $a$ under the natural 
projection 
$\pi_{d} : \mathbb{Z} \twoheadrightarrow \mathbb{Z}/p^{d+1}\mathbb{Z}$.
The embedding therefore implements $\psi_p \circ \pi_d$ rather than
$\psi_p$ itself — a finite-group character rather than a $p$-adic one.
In particular, the depth stack $d = 0, \ldots, D-1$ is faithful 
precisely on residues modulo $p^{D}$: integers congruent modulo 
$p^{D}$ are indistinguishable to the $p$-row.

\paragraph{Parameter choices.} In all experiments we fix
\[
    \mathcal{P}_0 = \{3, 5, 7, 11, 13, 17, 19, 23, 29,  \quad\quad\]
    \[31,37, 
    41, 43, 47, 53, 59\},
\]
the 16 primes from $3$ to $59$, excluding $2$\footnote{For $p=2$, 
$\sin(2\pi a/2) = 0$ for all integers $a$, so the sine feature 
carries no information. We exclude $p=2$ to avoid a structurally 
degenerate row.}, and depth cap $D = 6$, giving total embedding 
dimension $2 \times 16 \times 6 = 192$ and per-prime row dimension 
$4D = 24$ (the interleaved $(\sin, \cos)$ features of both $a$ and 
$b$ across all $D$ depths). This covers all odd prime factors of any 
odd modulus $N \leq 59$, and any squarefree composite whose prime factors 
lie in $\mathcal{P}_0$. Any task prime $p \notin \mathcal{P}_0$ 
would be invisible to the model.

\section{Theoretical Analysis}
\label{sec:theory}

The PFE embedding is equivariant by construction 
(Section~\ref{sec:construction}). Any linear map $W$ that respects 
this equivariance is an intertwiner. Schur's Lemma 
(Appendix~\ref{app:theory}, Theorem~\ref{thm:schur}) then imposes 
rigidity: any intertwiner between non-isomorphic irreducible 
representations must be zero.

Since $\mathbb{Z}/p^{d+1}\mathbb{Z}$ is abelian, every irreducible 
representation is one-dimensional 
(Appendix~\ref{app:theory}, Theorem~\ref{thm:irreps}), so in 
particular $\chi_1(a) = e^{2\pi i a/p^{d+1}}$ is an irrep. The 
block $\mathrm{PFE}_{p,d}$ implements this $\chi_1$ specifically. 
Across distinct primes or distinct depths, the groups 
$\mathbb{Z}/p^{d+1}\mathbb{Z}$ have different orders and are 
therefore non-isomorphic, so their $\chi_1$ characters are 
non-isomorphic irreps. By Schur's Lemma (Theorem~\ref{thm:schur}), 
any intertwiner between these blocks must be zero.

\subsection{Block-Diagonal Structure}

For each prime $p \in \mathcal{P}$ and depth $d = 0, \ldots, D-1$, 
the block $\mathrm{PFE}_{p,d}$ carries the character $\chi_1$ of 
$\mathbb{Z}/p^{d+1}\mathbb{Z}$, realized as a $2 \times 2$ rotation 
by angle $\theta_{p,d} = 2\pi/p^{d+1}$. The full embedding is 
doubly indexed:
\[
    \mathbf{h} = \bigoplus_{p \in \mathcal{P}} \bigoplus_{d=0}^{D-1} 
    \mathrm{PFE}_{p,d}(a) \;\in\; \mathbb{R}^{2 \cdot |\mathcal{P}| 
    \cdot D},
\]
where $\mathcal{P} \subseteq \mathcal{P}_0$. This gives rise to two 
distinct levels of block structure characterized by the following 
theorem.

\begin{theorem}[Block-Diagonal Decomposition of the Classifier]
\label{thm:block_diag}
Let $W$ be any linear map that is equivariant with respect to the 
action of $\mathbb{Z}/p^{d+1}\mathbb{Z}$ on each block 
$\mathrm{PFE}_{p,d}$. Then $W$ decomposes as
\[
    W = \bigoplus_{p \in \mathcal{P}} W_p, \qquad 
    W_p = \bigoplus_{d=0}^{D-1} W_{p,d},
\]
where each $W_{p,d}$ acts only within the $(p,d)$-subspace. In 
particular there is no coupling between blocks of distinct primes, 
and no coupling between blocks of distinct depths within the same 
prime.
\end{theorem}

\begin{proof}
We view each PFE block as the real two-dimensional realization 
of a complex character of the common additive group $\mathbb{Z}$. 
Specifically, $\chi_{p,d} : \mathbb{Z} \to S^1$ defined by 
$\chi_{p,d}(a) = e^{2\pi i a / p^{d+1}}$ factors through the 
quotient $\mathbb{Z} \twoheadrightarrow \mathbb{Z}/p^{d+1}\mathbb{Z}$. 
All blocks are therefore representations of the same group 
$\mathbb{Z}$, and two blocks $(p,d)$ and $(q,d')$ correspond to 
distinct characters of this common action whenever 
$p^{d+1} \neq q^{d'+1}$.

\textit{Distinct primes.} For $p \neq q$, the periods $p^{d+1}$
and $q^{d'+1}$ are powers of distinct primes, hence coprime and in 
particular distinct. Moreover, any group homomorphism 
$\mathbb{Z}/p^{d+1}\mathbb{Z} \to \mathbb{Z}/q^{d'+1}\mathbb{Z}$ 
is trivial, since the image order must divide both $p^{d+1}$ and 
$q^{d'+1}$, which are coprime. After complexification, the real block $(p,d)$ decomposes as 
$\chi_{p,d} \oplus \overline{\chi_{p,d}}$. A character and its 
conjugate have the same period, so every pairing of a summand of 
block $(p,d)$ with a summand of block $(q,d')$ involves two 
characters of $\mathbb{Z}$ with distinct periods, hence 
non-isomorphic one-dimensional representations.
By Schur's Lemma (Theorem~\ref{thm:schur}), any 
$\mathbb{Z}$-equivariant linear map between non-isomorphic 
irreducible representations must be zero. Hence $W_{p,d;\,q,d'} 
= 0$ for all $p \neq q$.

\textit{Distinct depths, same prime.} For fixed $p$ and $d \neq 
d'$, the periods $p^{d+1}$ and $p^{d'+1}$ differ, so each of $\chi_{p,d}, \overline{\chi_{p,d}}$ is a character of 
$\mathbb{Z}$ with a period different from that of each of 
$\chi_{p,d'}, \overline{\chi_{p,d'}}$; all four cross-pairings 
are therefore non-isomorphic as one-dimensional complex 
representations.
Schur's Lemma again forces any equivariant linear map between them 
to zero. Hence $W_{p,d;\,p,d'} = 0$ for $d \neq d'$.

Combining both cases, $W$ is block-diagonal with one independent 
block per $(p,d)$ pair.
\end{proof}

\begin{remark}[Strength of the two independence guarantees]
\label{rem:depth}
The between-prime and within-prime independence are of qualitatively 
different algebraic character. For distinct primes $p \neq q$, the CRT decomposition 
(Theorem~\ref{thm:crt}) separates the prime-local residue 
coordinates into independent coprime components. Any group 
homomorphism from $\mathbb{Z}/p^{d+1}\mathbb{Z}$ to 
$\mathbb{Z}/q^{d'+1}\mathbb{Z}$ is trivial since the image 
order must divide both $p^{d+1}$ and $q^{d'+1}$, which are 
coprime. This is strong independence. 
For distinct depths $d \neq d'$ within the same prime $p$, there is 
a natural surjection 
$\mathbb{Z}/p^{d'+1}\mathbb{Z} \twoheadrightarrow 
\mathbb{Z}/p^{d+1}\mathbb{Z}$ for $d' > d$, so the deeper block 
contains strictly more information. Schur's Lemma forces 
$W_{p,d;\,p,d'} = 0$ but does not force the network to distribute 
importance across depth levels — it only forbids cross-depth linear 
coupling.
\end{remark}

With respect to the basis 
$\{\mathbf{e}_{p,d,\cos}, \mathbf{e}_{p,d,\sin}\}$, the action of 
$\mathbb{Z}/p^{d+1}\mathbb{Z}$ on $\mathrm{PFE}_{p,d}$ is realized 
in $\mathrm{GL}(\mathbb{R}^2)$ as
\[
    R(\theta_{p,d}) = \begin{bmatrix} \cos\theta_{p,d} & 
    -\sin\theta_{p,d} \\ \sin\theta_{p,d} & \cos\theta_{p,d} 
    \end{bmatrix}, \qquad \theta_{p,d} = \frac{2\pi}{p^{d+1}}.
\]
Within a block, equivariance means that $W_{p,d}$ commutes with 
$R(\theta_{p,d})$; since $p$ is odd, $\chi_{p,d} \not\cong 
\overline{\chi_{p,d}}$, and the commutant consists exactly of the 
scaled rotations $W_{p,d} = \lambda_{p,d}\, R(\varphi_{p,d})$ with 
$\lambda_{p,d} \geq 0$ — the real form of multiplication by a 
complex scalar, as in the second half of Schur's Lemma 
(Theorem~\ref{thm:schur}).
The full weight matrix $W$ therefore has the nested block structure 
shown in Figure~\ref{fig:block_diag}: block-diagonal across 
$(p,d)$ pairs, with an otherwise arbitrary scaled rotation in 
each block.

\begin{figure}[t]
\vskip 0.2in
\begin{center}
\renewcommand{\arraystretch}{1.5}
$W$=\(
\begin{bmatrix}
    W_3&&&&\\
    &W_5&&&\\
    &&W_7&&\\
    &&&W_{11}&\\
    &&&&\ddots
\end{bmatrix}
\)
\end{center}
where\\
\begin{center}
\renewcommand{\arraystretch}{1.8}
\(
W_p\!:=\!\begin{bmatrix}
    W_{p,0}&&&\\
    &W_{p,1}&&\\
    &&\ddots&\\
        &&&W_{p,D-1}
\end{bmatrix},
\qquad
W_{p,d} \;=\; \lambda_{p,d}\, R(\varphi_{p,d})
\)
\caption{Nested block-diagonal structure of an equivariant linear 
map $W$. Each block $W_{p,d}$ acts only within its 
$(p,d)$-subspace, with no coupling across primes or depths, and 
is a scaled rotation $\lambda_{p,d}\,R(\varphi_{p,d})$ — the 
free parameters of $W$. (The fixed group action on the 
$(p,d)$-block is the rotation $R(\theta_{p,d})$ with 
$\theta_{p,d} = 2\pi/p^{d+1}$.)}
\label{fig:block_diag}
\end{center}
\vskip -0.2in
\end{figure}

\begin{remark}
The question of \emph{why} gradient descent converges to an 
equivariant solution rather than breaking symmetry remains open 
and is left for future work.
\end{remark}

\section{Experiments}
\label{sec:experiments}

We design two ablation experiments to empirically validate the 
block-diagonal decomposition predicted by 
Theorem~\ref{thm:block_diag}. Both experiments follow the same 
methodology: train a PFE-based classifier on a modular arithmetic 
task, then systematically zero out individual prime rows of the 
embedding and measure the resulting accuracy drop. A row is 
\textbf{task-relevant} if its prime divides the modulus, and 
\textbf{task-irrelevant} otherwise. The key prediction is that 
task-relevant rows cause large accuracy drops when ablated, while 
task-irrelevant rows cause near-zero drops.

\paragraph{Experimental setup.} All models share the same 
architecture: a shared per-prime encoder (two linear layers 
$24 \to 64 \to 32$, ReLU activations) processes each prime 
row independently, and a two-layer classifier 
($|\mathcal{P}| \times 32 \to 128 \to N$, ReLU then linear) 
maps the concatenated row representations to $N$ output 
logits. The PFE features are deterministic and non-trainable. 
Each prime row contains the interleaved $(\sin, \cos)$ 
features of both $a$ and $b$ across all $D$ digit levels, 
giving row dimension $4D = 24$. Models are trained with Adam 
($\mathrm{lr} = 3 \times 10^{-3}$, default $\beta_1, \beta_2$) 
and cross-entropy loss on $80{,}000$ uniformly sampled 
$(a, b)$ pairs with $a, b \in \{0, \ldots, r-1\}$, using an 
80/20 train/test split (64,000 train, 16,000 test). 
Datasets are constructed from the complete set of distinct 
$(a, b)$ pairs; the 80/20 split is applied to unique pairs 
before any sampling, with zero train/test overlap verified
for every configuration.

Experiment~1 trains for 25 epochs with batch size 1024; 
Experiment~2 trains for 40 epochs to accommodate the larger 
output space ($N$ up to $385$). All experiments use a single 
fixed random seed (42); variance across the sweep was 
negligible. Row ablation sets all $4D$ embedding dimensions 
of prime $p$ to zero — zeroing both the $a$ and $b$ features 
for that prime across all depths — and evaluates the frozen 
model on the held-out test set without retraining. The 
convergence threshold is $0.85$ for Experiment~1 and $0.70$ 
for Experiment~2. No error bars are reported as all 
experiments use a single seed; the sweep covers 35 
configurations for Experiment~1 and 50 for Experiment~2, 
providing breadth in place of repeated trials.

\subsection{Experiment 1: Prime Specialization on Single-Prime Tasks}

\paragraph{Setup.} We train on $(a + b) \bmod p$ for each prime $p$ 
in a fixed subset $\mathcal{P}$ of $\mathcal{P}_0$, varying 
$|\mathcal{P}| \in \{4, 6, 8, 10, 12, 14, 16\}$ and input range 
$r \in \{100, 500, 1000, 2000, 4000\}$. For each configuration we 
train a separate model per task prime, ablate each row, and record 
the diagonal drop (ablating the task-prime row) and the off-diagonal 
drop (mean drop when ablating any other row). A model is considered \textbf{converged} if it achieves 
test accuracy above $0.85$ at any point during training. 
The diagonal drop $\Delta_\text{diag}$ is defined as the 
decrease in test accuracy when the task-prime row is zeroed; 
the off-diagonal drop $\Delta_\text{off}$ is the mean 
decrease across all other prime rows.

\paragraph{Results.} Across all configurations, the diagonal drop 
is substantially larger than the off-diagonal drop. At basis sizes 
$|\mathcal{P}| \leq 8$, the off-diagonal drop is statistically 
indistinguishable from zero in the majority of configurations 
(below a 2-standard-error bound of $0.011$--$0.032$ depending on 
test-set size), indicating near-perfect prime isolation. At 
$|\mathcal{P}| \geq 10$, off-diagonal drops are reliably above 
the noise floor, ranging from $0.03$ to $0.09$ across configurations, 
while diagonal drops remain at $0.82$--$0.92$.
Convergence above 
$85\%$ test accuracy is consistent across all configurations once 
the input range exceeds the task prime. At small input ranges 
($r = 100$) with large prime sets, selectivity is somewhat reduced 
— a finite-sample effect when the input range is comparable to the 
task prime. These results are consistent with the block-diagonal 
structure predicted by Theorem~\ref{thm:block_diag}: task-relevant 
channels dominate accuracy, and task-irrelevant channels are 
effectively inert.
\begin{figure}[t]
    \includegraphics[width=\linewidth]{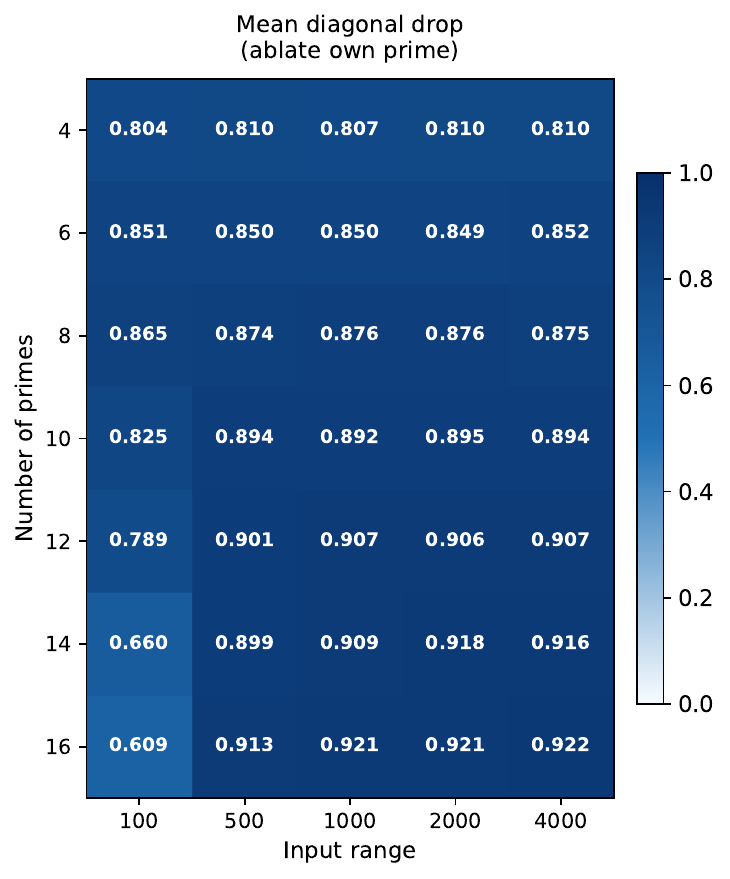}
\caption{Experiment~1. Mean diagonal drop 
(ablate own prime). Rows index $|\mathcal{P}|$; columns index 
input range $r$.}
\label{fig:exp1_diag}
\end{figure}

\begin{figure}[t]
    \includegraphics[width=\linewidth]{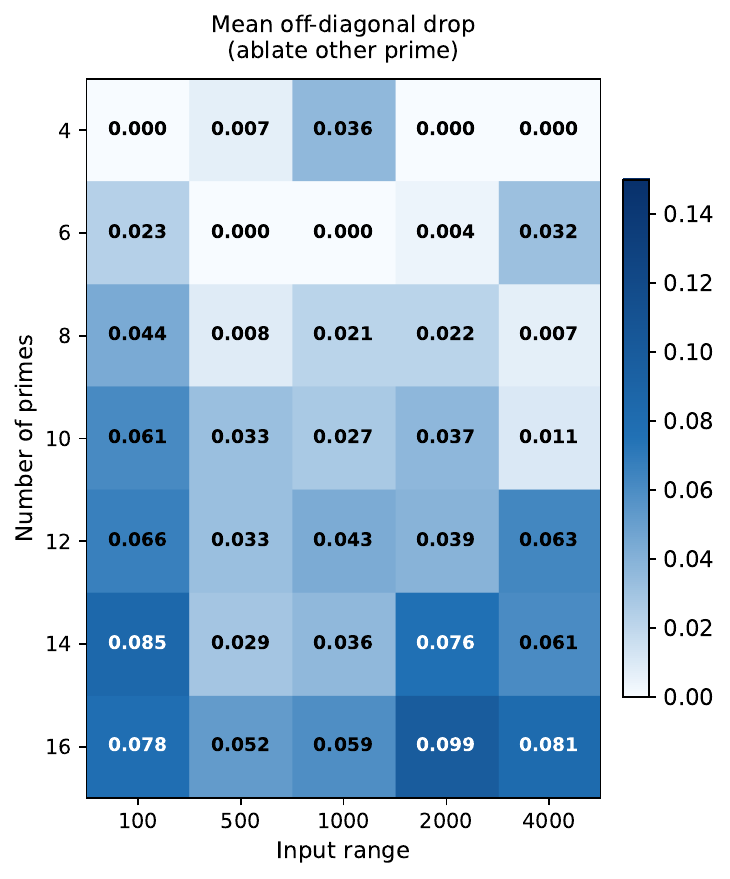}
\caption{Experiment~1. Mean off-diagonal drop 
(ablate other prime). Rows index $|\mathcal{P}|$; columns index 
input range $r$.}
\end{figure}

\begin{figure}[t]
    \includegraphics[width=\linewidth]{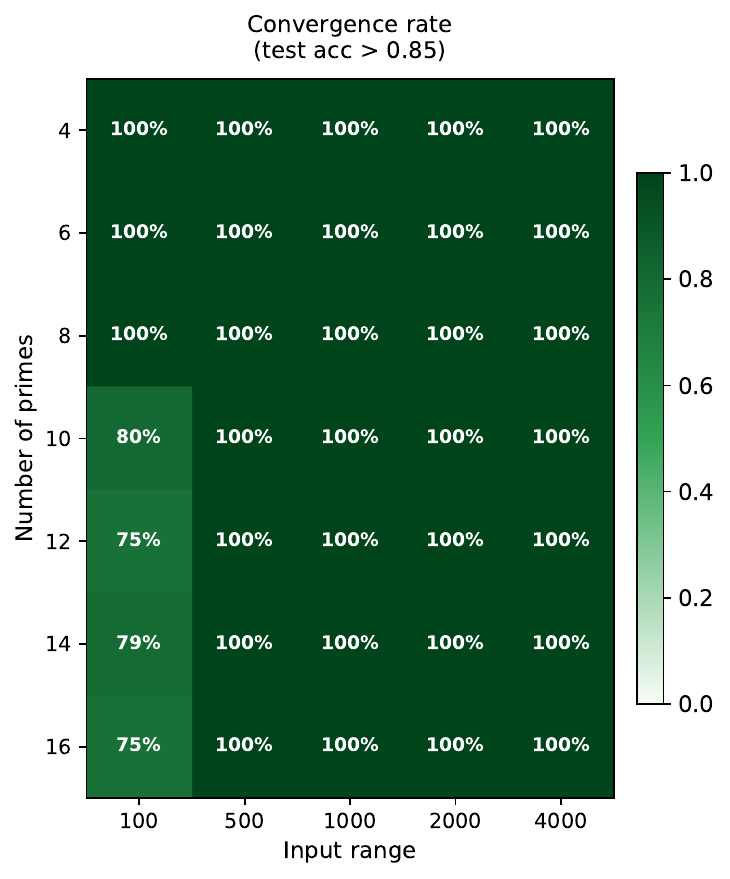}
\caption{Experiment~1. Convergence rate 
(test acc $> 0.85$). Rows index $|\mathcal{P}|$; columns index 
input range $r$.}
\end{figure}

\paragraph{Matched control: PFE-Shuffled ablation profiles.}
To verify that the clean diagonal/off-diagonal separation is a 
consequence of the equivariant block structure rather than a 
trivial property of any model trained on modular arithmetic, we 
repeat the row-ablation protocol on a PFE-Shuffled variant: the 
same $(\cos,\sin)$ features at the same prime frequencies, but 
with the prime-to-row correspondence destroyed by a fixed random 
permutation applied at initialization. The results are shown in 
Figures~\ref{fig:shuf_diag} and~\ref{fig:shuf_off}.

PFE-Shuffled's diagonal drops range from $0.10$ to $0.55$ — 
substantially lower than PFE's $0.54$--$0.93$ —
and do not 
exhibit the consistent monotone structure visible in the PFE 
heatmap. More strikingly, PFE-Shuffled's off-diagonal drops 
range from $0.14$ to $0.46$: ablating any row, whether or not 
it corresponds to the task prime, causes a comparably large 
accuracy drop. The model has distributed the task-relevant 
information across all rows rather than routing it into the 
task-prime channel. This is the opposite of what PFE produces, 
and directly answers the matched-control concern: the clean 
ablation structure in Experiment~1 is not a property that any 
sufficiently expressive model exhibits, but a consequence of 
the equivariant block-diagonal structure that PFE provides and 
PFE-Shuffled destroys.

\begin{figure}[t]
    \includegraphics[width=\linewidth]{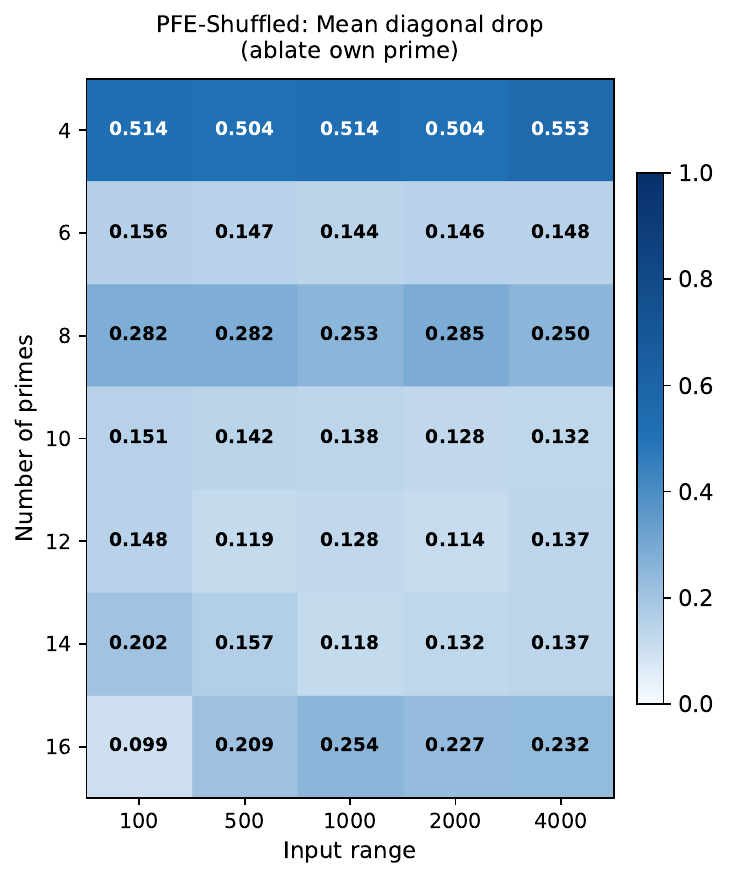}
\caption{PFE-Shuffled matched control: mean diagonal drop 
(ablate own prime row). Rows index $|\mathcal{P}|$; columns index 
input range $r$.}
\label{fig:shuf_diag}
\end{figure}

\begin{figure}[t]
    \includegraphics[width=\linewidth]{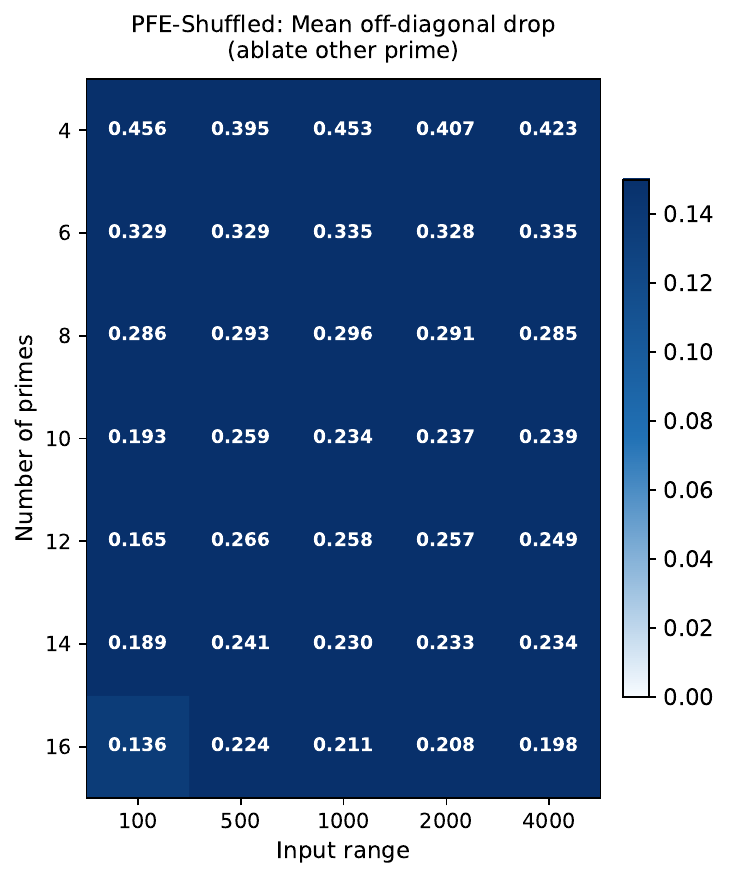}
\caption{PFE-Shuffled matched control: mean off-diagonal drop 
(ablate other prime row). Rows index $|\mathcal{P}|$; columns index 
input range $r$.}
\label{fig:shuf_off}
\end{figure}

\subsection{Experiment 2: CRT Decomposition on Composite Moduli}

\paragraph{Setup.} We train on $(a + b) \bmod N$ for squarefree 
composite moduli $N$, covering two-factor composites 
$N \in \{15, 21, 33, 35, 55, 77\}$ and three-factor composites 
$N \in \{105, 165, 231, 385\}$, each formed as a product of primes 
from $\{3, 5, 7, 11\}$, embedded within the full prime basis 
$\mathcal{P} = \{3, 5, 7, 11, 13, 17, 19, 23\}$. We sweep over 
input ranges $r \in \{100, 500, 1000, 2000, 4000\}$. The CRT 
prediction is that for $N = p_1 \cdots p_k$, the $k$ rows 
corresponding to the prime factors of $N$ should activate, while 
all remaining rows remain near-zero. The factor drop is the mean accuracy decrease when ablating 
a prime $p$ that divides $N$; the nonfactor drop is the 
mean decrease when ablating a prime $p \nmid N$. 
\begin{figure}[t]
    \includegraphics[width=\linewidth]{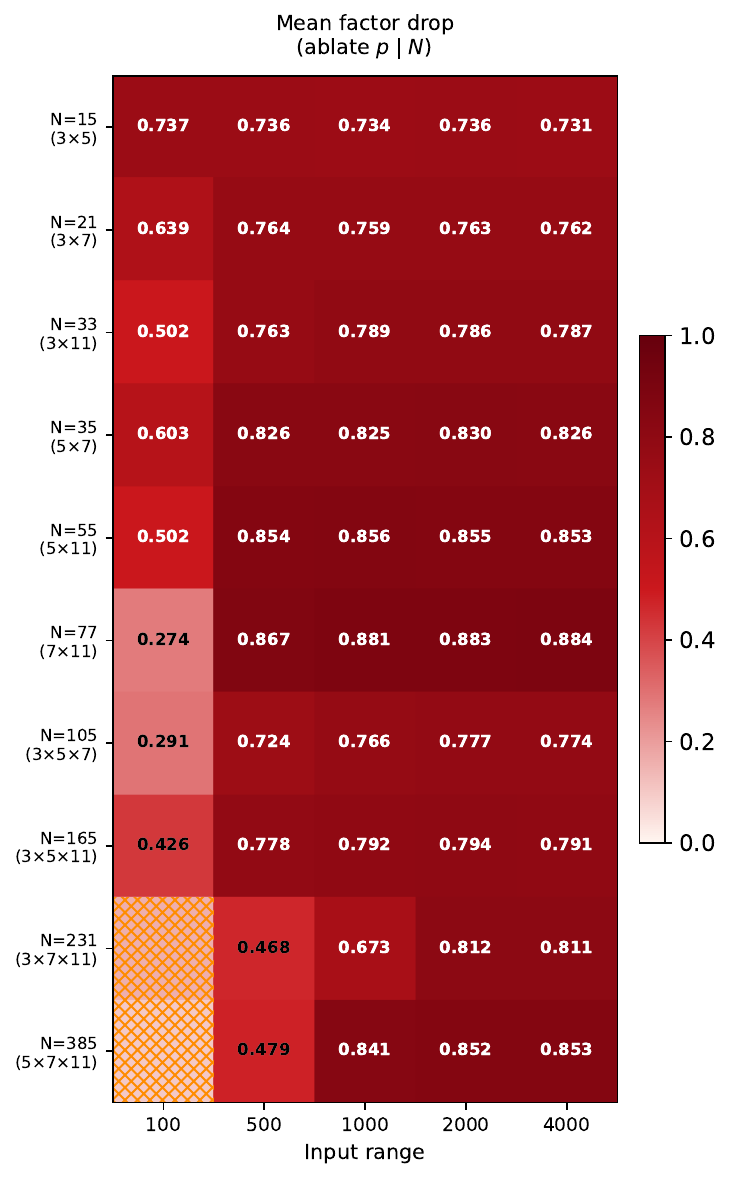}
\caption{Experiment 2. Mean factor drop. Rows index $N$; columns index $r$.}
\end{figure}

\paragraph{Results.} The results confirm the CRT prediction. For all 
non-degenerate squarefree composites and sufficiently large input 
ranges, the nonfactor drop is statistically indistinguishable from 
zero in 77\% of configurations, while factor drops range from 
$0.29$ to $0.88$ depending on $N$ and $r$. Perfect test accuracy 
($1.00$) is achieved across all non-degenerate configurations at 
$r \geq 500$, confirming that the block-diagonal routing is 
computationally complete. For two-factor composites, factor-channel 
selectivity is clean at all $r \geq 500$. Three-factor composites 
achieve perfect accuracy at $r \geq 500$ as well, though 
factor-channel activation is reduced at $r = 100$, where the input 
range is comparable to $N$ itself — a finite-sample phenomenon 
rather than a structural failure. The configurations $N = 231$ and 
$N = 385$ at $r = 100$ are excluded as degenerate: with 
$\max(a+b) = 198 < 231$, the modular sum never wraps.
\begin{figure}[t]
    \includegraphics[width=0.95\linewidth]{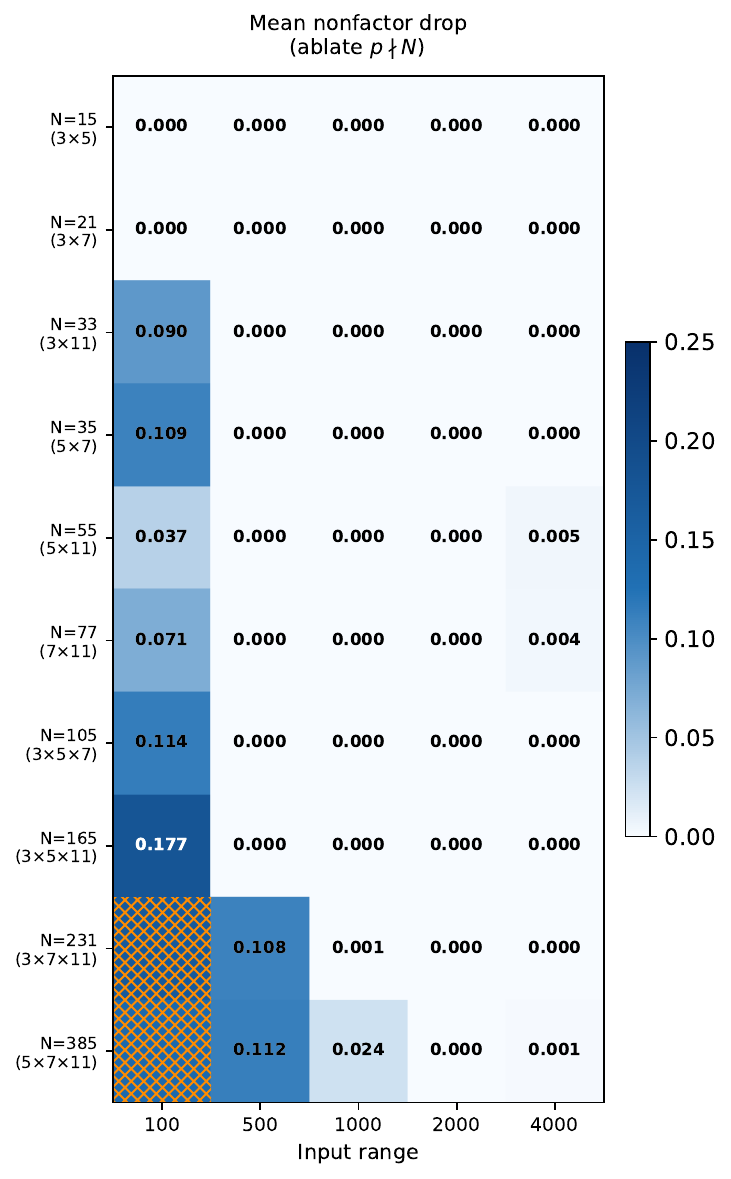}
\caption{Experiment 2. Mean nonfactor drop. Rows index $N$; columns index $r$.}
\end{figure}
\begin{figure}[t]
    \includegraphics[width=0.95\linewidth]{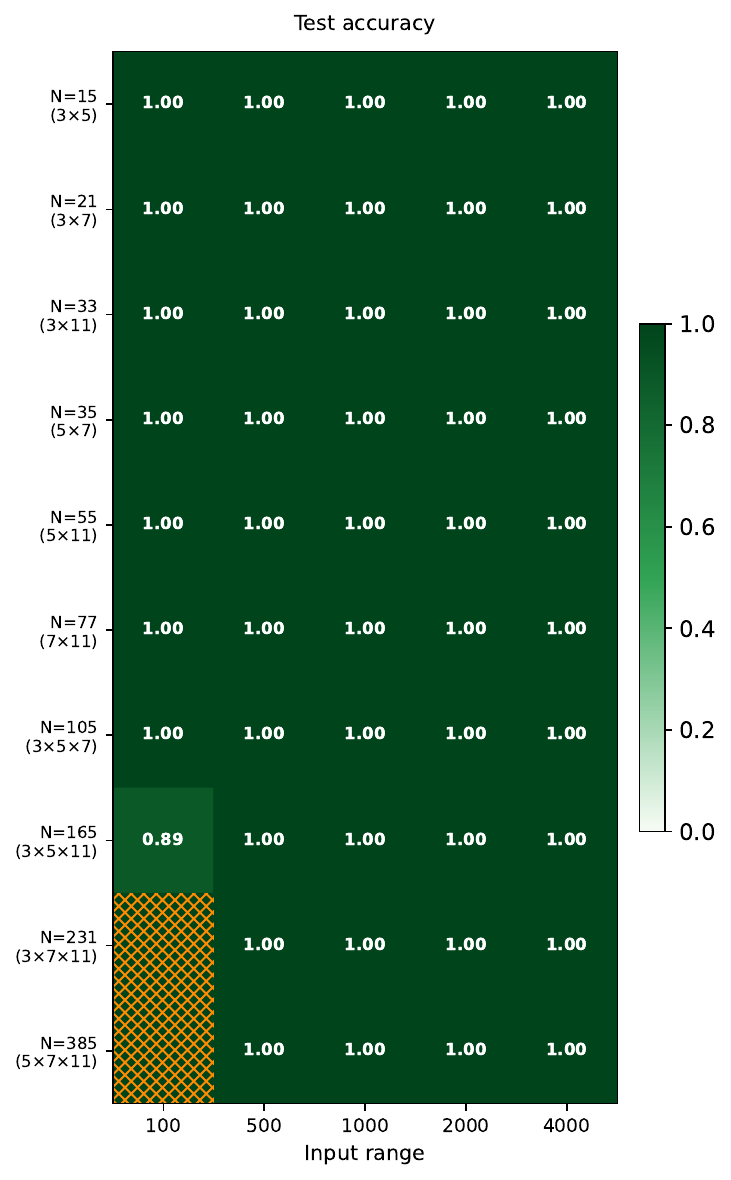}
\caption{Experiment 2. Test accuracy. Rows index $N$; columns index $r$.}
\end{figure}

\section{Discussion}

\paragraph{What the experiments confirm.} The ablation results 
provide empirical validation of Theorem~\ref{thm:block_diag}. 
Each prime channel operates independently, computation is routed 
through the particular channels relevant to the task modulus, and 
ablating an irrelevant channel has negligible effect. This is a 
falsifiable prediction of the representation theory confirmed across 
a systematic sweep of moduli, prime counts, and input ranges.

\paragraph{Structure from initialization.} 
Prior mechanistic work has shown that models trained with 
one-hot inputs on modular addition must reconstruct a Fourier 
basis from scratch through gradient descent 
\cite{nanda2023progress, gromov2023grokking} — confirming that 
the right algebraic structure exists but must be discovered 
implicitly. PFE provides this structure from initialization: 
task-relevant channels are precisely those predicted by the CRT 
to carry modular structure, and ablating them causes large 
accuracy drops while ablating irrelevant channels has negligible 
effect.

\paragraph{The principled basis argument.} The choice of prime 
frequencies is mathematically motivated rather than heuristic. 
By Ostrowski's theorem 
(Appendix~\ref{app:ostrowski}, Theorem~\ref{thm:ostrowski}), 
composite bases such as base-10 do not isolate prime-local 
structure into independent valuation channels — the $2$-adic 
and $5$-adic periodicities of base-10 share a single channel, 
leaving their separation as an implicit task for the network.
The adelic character factorization 
(Appendix~\ref{app:pontryagin}, Theorem~\ref{thm:adelic}) 
establishes that the Fourier basis on $\mathbb{A}_\mathbb{Q}$ 
factorizes into independent local components indexed by primes. 
PFE implements these components directly, making the prime 
basis a principled rather than arbitrary choice for modular 
arithmetic tasks whose structure is organized prime-locally 
by the Chinese Remainder Theorem.

\paragraph{Open questions.} Two directions remain open. First, 
why does gradient descent converge to an equivariant solution? The 
loss is invariant under simultaneous rotation but gradient descent 
is not guaranteed to preserve this symmetry. Second, the within-prime 
depth structure is not fully understood. Schur's lemma forces zero 
linear coupling across depth levels within the same prime, but the 
depth levels are not genuinely independent: there is a natural 
surjection $\mathbb{Z}/p^{d'+1}\mathbb{Z} \twoheadrightarrow 
\mathbb{Z}/p^{d+1}\mathbb{Z}$ for $d' > d$, so deeper blocks 
subsume the information in shallower ones. The theory predicts 
zero cross-depth coupling but says nothing about how gradient 
descent distributes importance across depth levels — a question 
whose answer would deepen our structural understanding of the 
embedding.

\section{Conclusion}

We have introduced Prime Fourier Embeddings (PFE), a numerical 
embedding whose frequency basis is derived from the adelic 
harmonic analysis of $\mathbb{Q}$. PFE implements the local 
$p$-adic characters as prime-indexed $(\cos, \sin)$ pairs, 
providing a pre-structured representation in which modular 
arithmetic reduces to selecting the relevant prime channel 
rather than discovering algebraic structure from scratch.

The central theoretical contribution is the Block-Diagonal 
Decomposition Theorem (Theorem~\ref{thm:block_diag}): any 
linear map equivariant with respect to the product action of 
$\prod_{p,d} \mathbb{Z}/p^{d+1}\mathbb{Z}$ on the PFE embedding 
must be block-diagonal, with one independent block per $(p,d)$ 
pair and zero coupling across blocks. This follows from viewing 
each PFE block as the real form of a character of the common 
additive group $\mathbb{Z}$, and applying Schur's lemma to the 
resulting non-isomorphic character components. For composite 
moduli, the Chinese Remainder Theorem predicts which prime-local 
blocks are relevant to the task. Both predictions are confirmed empirically: ablation studies confirm the block-diagonal prediction directly,
task-relevant channels show diagonal drops of $0.60$--$0.92$ while 
task-irrelevant channels are statistically inert in the majority
of configurations.

The broader implication is that embedding design for arithmetic 
tasks can be guided by the mathematics of the objects being 
represented. Where the structure of learned representations is 
typically inferred post-hoc, here it is characterized before 
training and verified through targeted ablations.


\section*{Impact Statement}

This paper presents work at the intersection of numerical 
representation and mechanistic interpretability. The central 
result — that any linear map equivariant with respect to PFE 
must be block-diagonal with one independent block per prime — 
is a provable structural constraint on network behavior, not a 
post-hoc observation. Grounding interpretability claims in 
representation theory, where the structure of learned 
representations can be predicted and falsified before training, 
offers a principled path toward more reliable and transparent 
numerical reasoning in neural networks — and more broadly, 
toward the kind of falsifiable, mathematically grounded theory 
of deep learning called for by 
\citet{simon2026scientifictheorydeeplearning}. We do not 
foresee specific near-term harms from this work.

\bibliography{references}
\bibliographystyle{icml2026}

\appendix
\onecolumn

\section{Theoretical Background}
\label{app:theory}

This appendix collects definitions and results used in the main body.
We assume familiarity with basic abstract algebra at the level of
\cite{lang2002algebra} and point-set topology at the level of
\cite{munkres2000topology}. Definitions are included where precision
matters for the proofs or where terminology varies across communities;
standard graduate material is cited rather than reproved.

\subsection{Algebraic Structures}

Standard references: \cite{lang2002algebra, hungerford1974algebra}.

\begin{definition}[Monoid]
A \textbf{monoid} is a set $M$ equipped with a binary operation 
$\cdot : M \times M \to M$ satisfying: (i) associativity: 
$(a \cdot b) \cdot c = a \cdot (b \cdot c)$ for all $a, b, c \in M$, 
and (ii) existence of an identity element $e_M \in M$ with 
$e_M \cdot m = m \cdot e_M = m$ for all $m \in M$.
\end{definition}

\begin{definition}[Group]
A \textbf{group} is a monoid $(G, \cdot)$ in which every element 
has an inverse: for each $g \in G$ there exists $g^{-1} \in G$ with 
$g \cdot g^{-1} = g^{-1} \cdot g = e_G$. A group is \textbf{abelian} 
if additionally $g \cdot h = h \cdot g$ for all $g, h \in G$.
\end{definition}

\begin{definition}[Ring]
A \textbf{ring} is a set $R$ equipped with two binary operations 
$+$ and $\cdot$ such that $(R, +)$ is an abelian group with identity 
$0_R$, $(R, \cdot)$ is a monoid with identity $1_R$, and 
multiplication distributes over addition. A ring is 
\textbf{commutative} if additionally $a \cdot b = b \cdot a$.
\end{definition}

\begin{definition}[Ring Homomorphism and Isomorphism]
A \textbf{ring homomorphism} is a map $\phi : R \to S$ satisfying 
$\phi(a + b) = \phi(a) + \phi(b)$, $\phi(a \cdot b) = \phi(a) \cdot 
\phi(b)$, and $\phi(1_R) = 1_S$. A bijective ring homomorphism is a 
\textbf{ring isomorphism}, written $R \cong S$.
\end{definition}

\subsection{Topological Notions}
\label{app:topology}

Standard reference: \cite{munkres2000topology}.

\begin{definition}[Topological Space]
A \textbf{topological space} is a set $X$ together with a collection 
$\tau$ of subsets called \textbf{open sets}, satisfying: 
(i) $\emptyset$ and $X$ are open; (ii) arbitrary unions of open sets 
are open; (iii) finite intersections of open sets are open.
\end{definition}

\begin{definition}[Discrete Topology]
The \textbf{discrete topology} on a set $X$ is the topology in which 
every subset of $X$ is open.
\end{definition}

\begin{definition}[Continuous Map]
A map $f : X \to Y$ is \textbf{continuous} if $f^{-1}(U)$ is open 
in $X$ for every open set $U$ in $Y$.
\end{definition}

\begin{definition}[Quotient Topology]
Let $X$ be a topological space and ${\sim}$ an equivalence relation. 
The \textbf{quotient space} $X/{\sim}$ carries the \textbf{quotient 
topology}: $U \subset X/{\sim}$ is open iff its preimage under the 
projection $\pi : X \to X/{\sim}$ is open in $X$.
\end{definition}

\begin{definition}[Hausdorff Space]
$X$ is \textbf{Hausdorff} if any two distinct points have disjoint 
open neighborhoods.
\end{definition}

\begin{definition}[Compact Space]
$X$ is \textbf{compact} if every open cover has a finite subcover.
\end{definition}

\begin{definition}[Locally Compact Space]
$X$ is \textbf{locally compact} if every point has a neighborhood 
contained in a compact subset.
\end{definition}

\subsection{Representations and Characters}

Standard references: \cite{serre1977linear, fulton1991representation, 
terras1999fourier}.

\begin{definition}[Linear Representation]
A \textbf{representation} of a group $G$ is a pair $(V, \rho)$ where 
$V$ is a finite-dimensional vector space over $\mathbb{C}$ and 
$\rho : G \to \mathrm{GL}(V)$ is a group homomorphism.
\end{definition}

\begin{definition}[Irreducible Representation]
A representation $(V, \rho)$ is \textbf{irreducible} (an 
\textbf{irrep}) if $V \neq 0$ and the only $G$-stable subspaces are 
$\{0\}$ and $V$. By Maschke's theorem, every finite-dimensional 
representation of a finite group decomposes as a direct sum of 
irreps \cite{serre1977linear}.
\end{definition}

\begin{definition}[Character]
A \textbf{character} of an abelian group $G$ is a group homomorphism 
$\chi : G \to S^1 \subset \mathbb{C}^\times$.
\end{definition}

\begin{definition}[Equivariant Map and Intertwiner 
\cite{fulton1991representation}]
Let $G$ act on $V$ and $W$ via representations $\rho$ and $\sigma$. 
A linear map $T : V \to W$ is \textbf{equivariant} (a 
\textbf{$G$-intertwiner}) if
\[
    T \circ \rho(g) = \sigma(g) \circ T \qquad \text{for all } g \in G.
\]
\end{definition}

\begin{theorem}[Schur's Lemma]
\label{thm:schur}
Let $(V, \rho)$ and $(W, \sigma)$ be irreducible representations of 
$G$. If $T : V \to W$ is an intertwiner, then either $T = 0$ or $T$ 
is an isomorphism. If $V = W$ over an algebraically closed field, 
then $T = \lambda I$ for some $\lambda \in \mathbb{C}$.
\end{theorem}

\begin{proof}
$\ker T$ is $G$-stable in $V$: by irreducibility, $\ker T = \{0\}$ 
or $V$. Similarly $\mathrm{im}\,T$ is $G$-stable in $W$: 
$\mathrm{im}\,T = \{0\}$ or $W$. These cases yield $T = 0$ or $T$ 
an isomorphism. When $V = W$ over an algebraically closed field, any 
eigenvalue $\lambda$ makes $T - \lambda I$ a non-invertible 
intertwiner, hence $T - \lambda I = 0$.
\end{proof}

\begin{theorem}[Irreps of $\mathbb{Z}/n\mathbb{Z}$]
\label{thm:irreps}
The irreducible complex representations of $\mathbb{Z}/n\mathbb{Z}$ 
are the $n$ characters
\[
    \chi_k(a) = e^{2\pi i k a / n}, \quad k = 0, 1, \ldots, n-1,
\]
forming a complete orthonormal basis for $L^2(\mathbb{Z}/n\mathbb{Z})$.
\end{theorem}

\begin{proof}
Since $\mathbb{Z}/n\mathbb{Z}$ is abelian, every irrep is 
one-dimensional by Theorem~\ref{thm:schur}. A one-dimensional 
representation $\rho : \mathbb{Z}/n\mathbb{Z} \to \mathbb{C}^\times$ 
must map elements of order dividing $n$ into $n$-th roots of unity 
$\subset S^1$, giving the $n$ characters $\chi_k$. 
Orthonormality follows from character orthogonality relations; see 
\cite{terras1999fourier}, Chapters~1--2.
\end{proof}

\subsection{The Chinese Remainder Theorem}

Standard reference: \cite{ireland1990classical}.

\begin{theorem}[Chinese Remainder Theorem]
\label{thm:crt}
Let $n = p_1^{a_1} \cdots p_k^{a_k}$. Then
\[
    \mathbb{Z}/n\mathbb{Z} \;\cong\; \mathbb{Z}/p_1^{a_1}\mathbb{Z} 
    \;\times\; \cdots \;\times\; \mathbb{Z}/p_k^{a_k}\mathbb{Z}.
\]
The residues $(a \bmod p_i^{a_i})$ are mutually independent across 
distinct primes, and the groups $\mathbb{Z}/p_i^{a_i}\mathbb{Z}$ 
are pairwise non-isomorphic.
\end{theorem}

\begin{proof}
See \cite{ireland1990classical}, Theorem~3.4.
\end{proof}

\subsection{Ostrowski's Theorem}
\label{app:ostrowski}

Standard reference: \cite{neukirch1999algebraic}.

\begin{theorem}[Ostrowski's Theorem]
\label{thm:ostrowski}
Every non-trivial absolute value on $\mathbb{Q}$ is equivalent to 
either the real absolute value $|\cdot|_\infty$ or the $p$-adic
absolute value $|\cdot|_p$ for some prime $p$.
\end{theorem}

\begin{proof}[Proof sketch]
See \cite{neukirch1999algebraic}, Chapter~II, \S3. The key dichotomy 
is between archimedean absolute values — all equivalent to 
$|\cdot|_\infty$ — and non-archimedean ones, satisfying the 
ultrametric inequality $|x+y| \leq \max(|x|,|y|)$, classified by 
primes via the $p$-adic valuations.
\end{proof}

In particular, the places of $\mathbb{Q}$ are indexed by the primes 
together with $\infty$: no absolute value on $\mathbb{Q}$ 
corresponds to a composite base, whose digit periodicities mix 
the valuations of its prime factors rather than defining a new 
place.

The completion of $\mathbb{Q}$ with respect to $|\cdot|_p$ is the 
\textbf{$p$-adic field} $\mathbb{Q}_p$, with ring of integers 
$\mathbb{Z}_p = \{x \in \mathbb{Q}_p : |x|_p \leq 1\}$ 
\cite{gouvea1997p}.

\subsection{Pontryagin Duality and the Adelic Character Factorization}
\label{app:pontryagin}

Standard references: \cite{ramakrishnan1999fourier, folland1995course}.

\begin{definition}[Topological Group]
A \textbf{topological group} is a group $G$ with a topology making 
multiplication and inversion continuous.
\end{definition}

\begin{definition}[Locally Compact Abelian Group]
A \textbf{locally compact abelian (LCA) group} is a topological group 
that is abelian, locally compact, and Hausdorff.
\end{definition}

\begin{definition}[Haar Measure]
A \textbf{Haar measure} on a LCA group $G$ is a nonzero 
translation-invariant Borel measure. Every LCA group admits a Haar 
measure, unique up to a positive scalar multiple \cite{folland1995course}.
\end{definition}

\begin{definition}[Restricted Product Topology]
Let $\{G_v\}_{v \in V}$ be locally compact Hausdorff groups, with 
compact open subgroups $K_v \subset G_v$ for all but finitely many 
$v$. The \textbf{restricted product} $\prod'_v (G_v, K_v)$ consists 
of tuples $(x_v)$ with $x_v \in K_v$ for all but finitely many $v$, 
with topology generated by sets 
$\prod_{v \in S} U_v \times \prod_{v \notin S} K_v$ for finite $S$ 
and open $U_v \subset G_v$. The adele ring 
$\mathbb{A}_\mathbb{Q} = \prod'_v (\mathbb{Q}_v, \mathbb{Z}_v)$ 
assembles all completions of $\mathbb{Q}$ — one archimedean place 
$\mathbb{R}$ and one $p$-adic place $\mathbb{Q}_p$ per prime — under 
this topology.
\end{definition}

The groups $\mathbb{R}$, $\mathbb{Q}_p$, $S^1 = \mathbb{R}/\mathbb{Z}$, 
finite groups under the discrete topology, and $\mathbb{A}_\mathbb{Q}$ 
are all LCA groups of interest in this paper.

\begin{definition}[Character and Pontryagin Dual]
A \textbf{character} of a LCA group $G$ is a continuous homomorphism 
$\chi : G \to S^1$. The group of all characters under pointwise 
multiplication is the \textbf{Pontryagin dual} $\widehat{G}$, itself 
a LCA group.
\end{definition}

\begin{theorem}[Pontryagin Duality \cite{folland1995course}]
\label{thm:pontryagin}
For any LCA group $G$, the map $G \to \widehat{\widehat{G}}$ sending 
$g \mapsto (\chi \mapsto \chi(g))$ is a topological group 
isomorphism.
\end{theorem}

\begin{proof}[Proof sketch]
See \cite{folland1995course}, Chapter~4, Theorem~4.31. The key 
ingredients are Haar measure and the Fourier inversion theorem on 
LCA groups.
\end{proof}

Applied to the adele ring: $\mathbb{A}_\mathbb{Q}$ is self-dual 
($\widehat{\mathbb{A}_\mathbb{Q}} \cong \mathbb{A}_\mathbb{Q}$), 
and this self-duality, together with the restricted product 
structure of ${\mathbb{A}_\mathbb{Q}}$, implies that any 
character factorizes into independent local components:

\begin{theorem}[Factorization of Adelic Characters
\cite{ramakrishnan1999fourier}]
\label{thm:adelic}
Any character $\psi : \mathbb{A}_\mathbb{Q} \to S^1$ 
factorizes as a product of local characters,
\[
    \psi(x) = \prod_{v} \psi_v(x_v),
\]
where the product runs over all places $v$ of $\mathbb{Q}$ 
(one archimedean, one $p$-adic per prime), and $\psi_v = 1$ 
for all but finitely many $v$. In particular, the standard 
choice $\psi_\infty(x_\infty) = e^{-2\pi i x_\infty}$ and 
$\psi_p(x_p) = e^{2\pi i \{x_p\}_p}$ gives a factorization 
into independent prime-local components.
\end{theorem}
\begin{proof}[Proof sketch]
See \cite{ramakrishnan1999fourier}, Chapters~4--5. The 
factorization follows from the restricted product structure 
of $\mathbb{A}_\mathbb{Q}$ and the fact that any character 
of a restricted product group decomposes into local 
characters that are trivial on the compact open subgroups 
$\mathbb{Z}_p$ for all but finitely many primes.
\end{proof}
\newpage
\section{Additional Experimental Figures}
\label{app:figures}
\subsection{Experiment 1: Per-Prime Ablation Profiles}
\begin{figure}[H]
\begin{center}
\centerline{\includegraphics[width=0.9\linewidth]{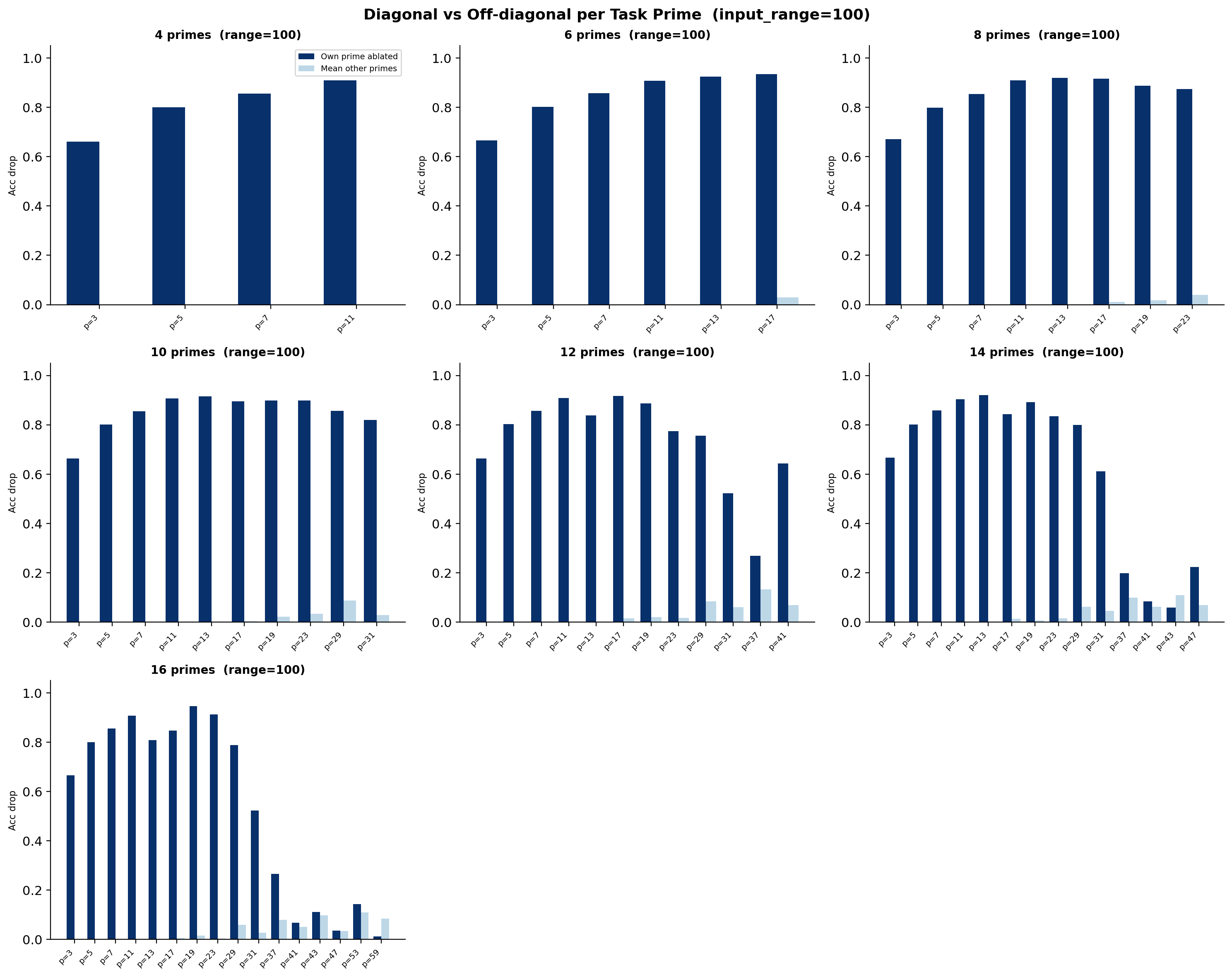}}
\caption{Per-prime ablation profiles, Experiment~1, $r = 100$.}
\label{fig:app_exp1_r100}
\end{center}
\end{figure}
\begin{figure}[H]
\begin{center}
\centerline{\includegraphics[width=\linewidth]{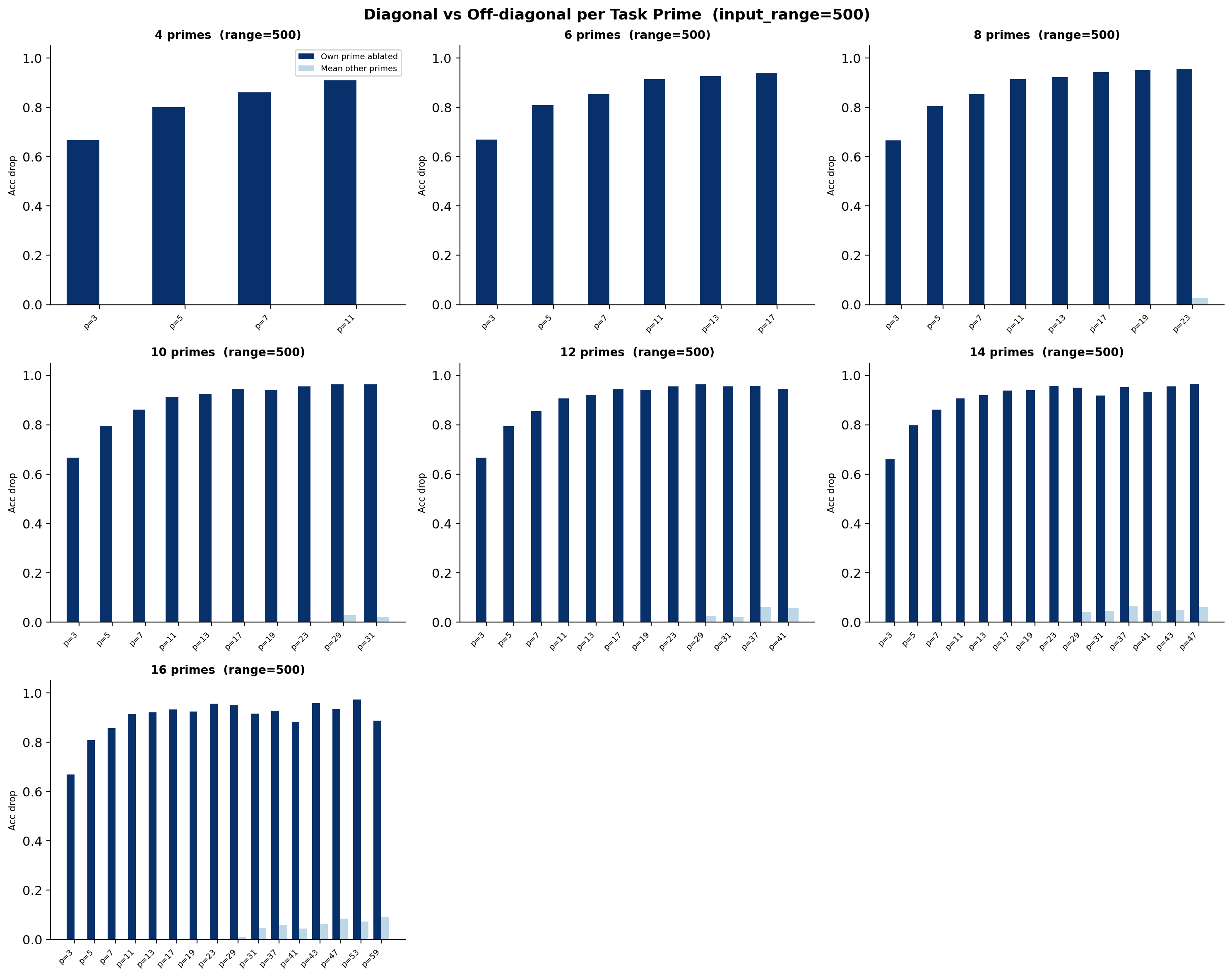}}
\caption{Per-prime ablation profiles, Experiment~1, $r = 500$.}
\label{fig:app_exp1_r500}
\end{center}
\end{figure}

\begin{figure}[H]
\vskip 0.2in
\begin{center}
\centerline{\includegraphics[width=\linewidth]{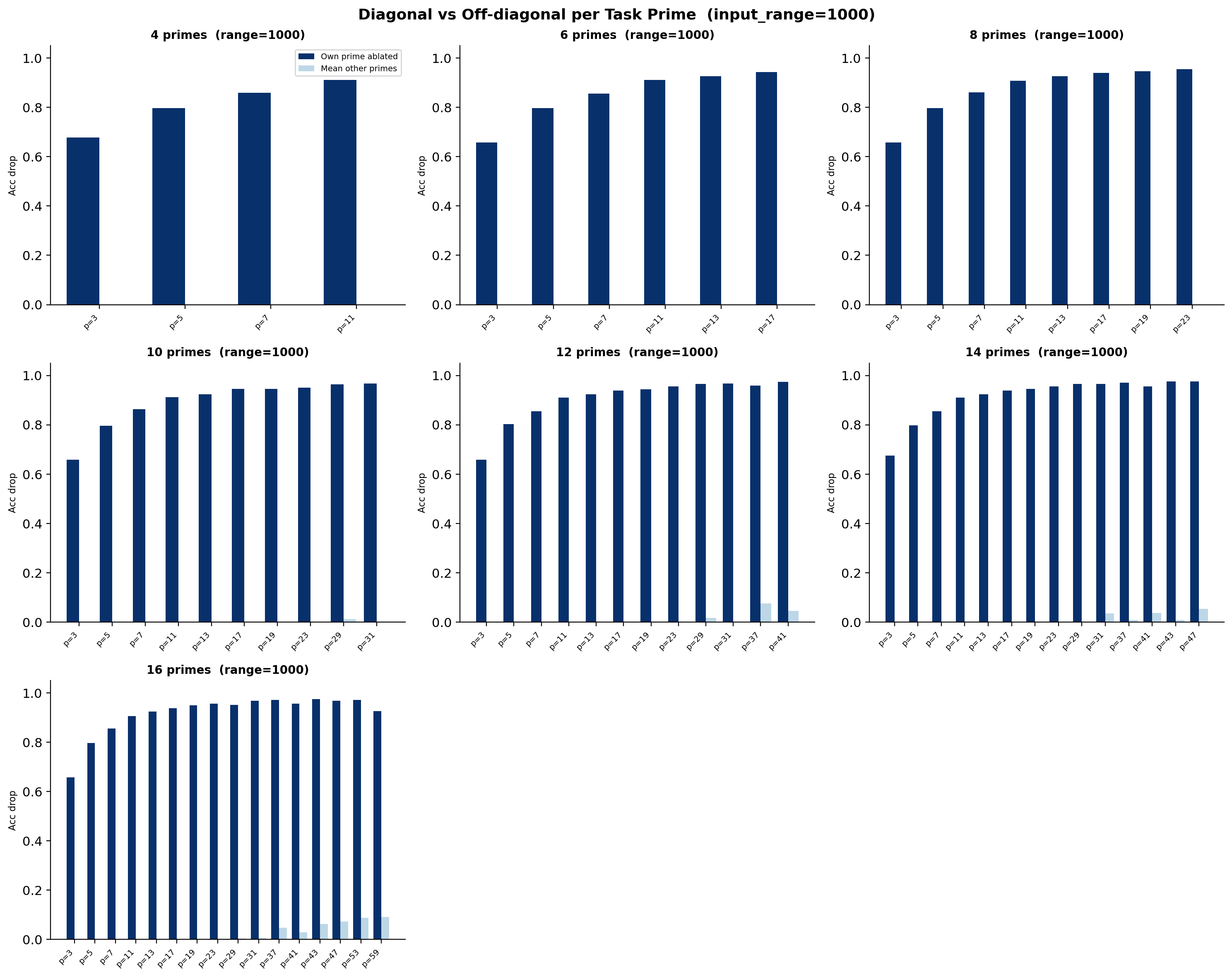}}
\caption{Per-prime ablation profiles, Experiment~1, $r = 1000$.}
\label{fig:app_exp1_r1000}
\end{center}
\end{figure}

\begin{figure}[H]
\vskip 0.2in
\begin{center}
\centerline{\includegraphics[width=\linewidth]{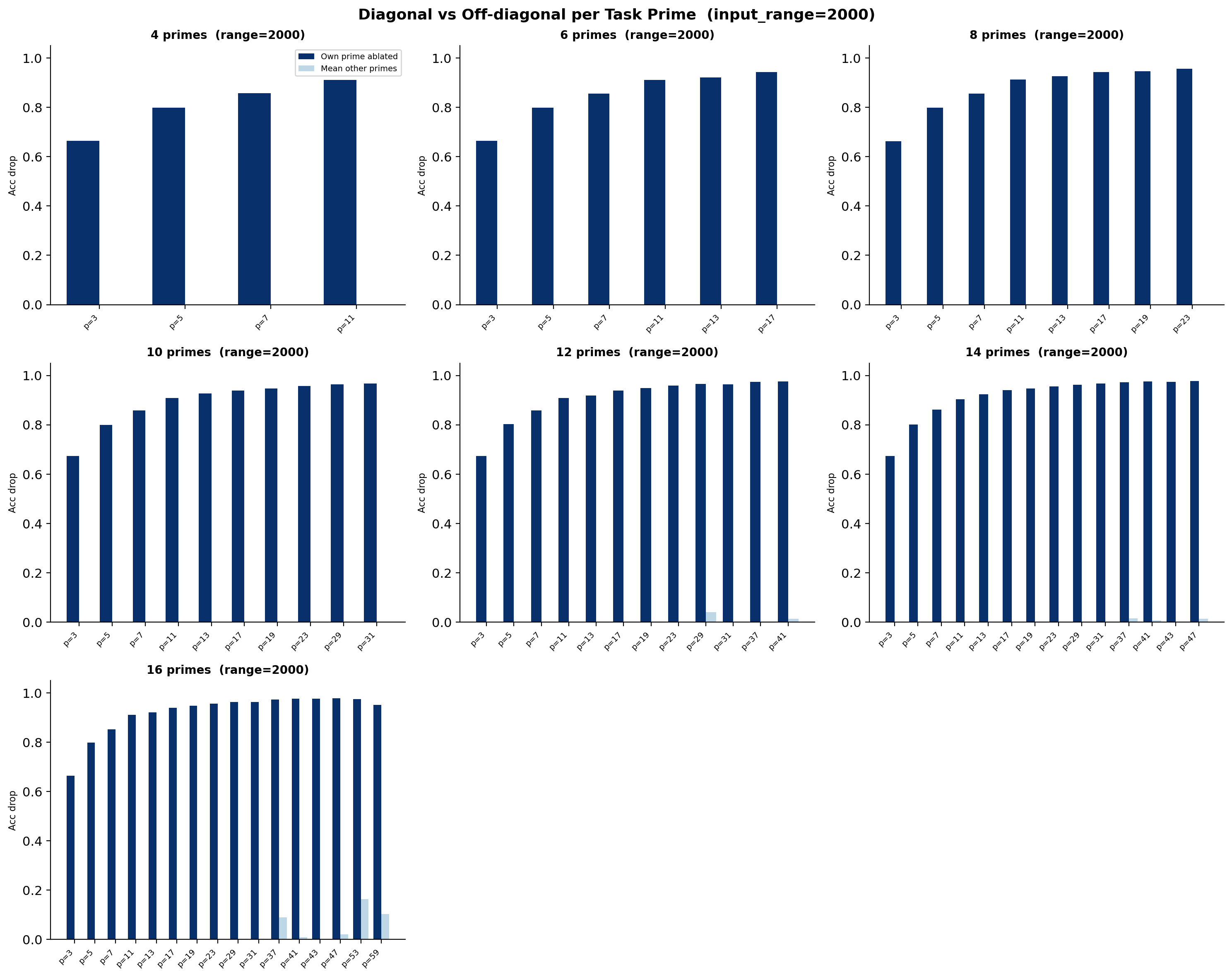}}
\caption{Per-prime ablation profiles, Experiment~1, $r = 2000$.}
\label{fig:app_exp1_r2000}
\end{center}
\end{figure}

\begin{figure}[H]
\vskip 0.2in
\begin{center}
\centerline{\includegraphics[width=\linewidth]{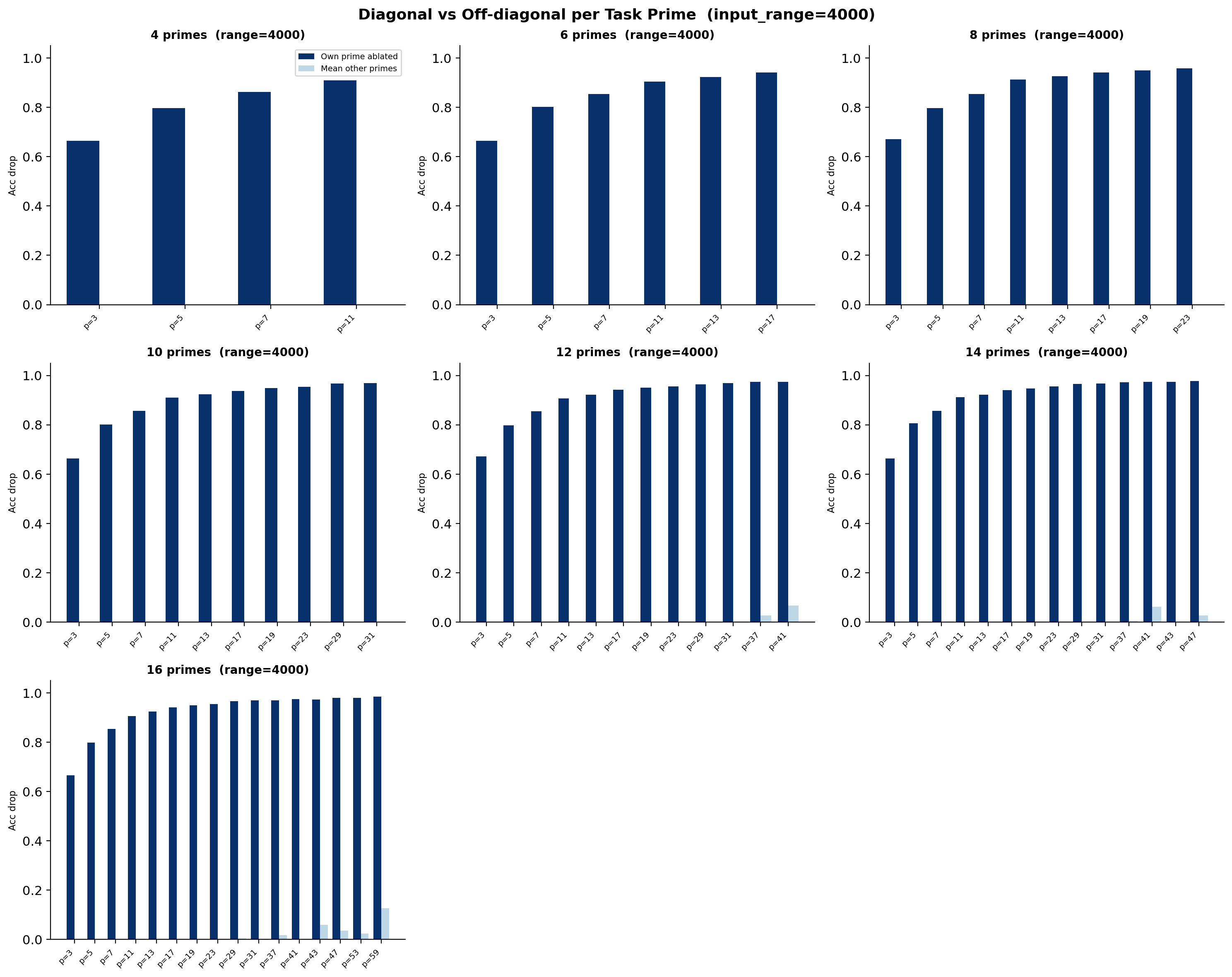}}
\caption{Per-prime ablation profiles, Experiment~1, $r = 4000$.}
\label{fig:app_exp1_r4000}
\end{center}
\end{figure}
\clearpage
\subsection{Experiment 1: Supplementary Diagnostic Plots}

\begin{figure}[H]
\vskip 0.2in
\begin{center}
\centerline{\includegraphics[width=\linewidth]{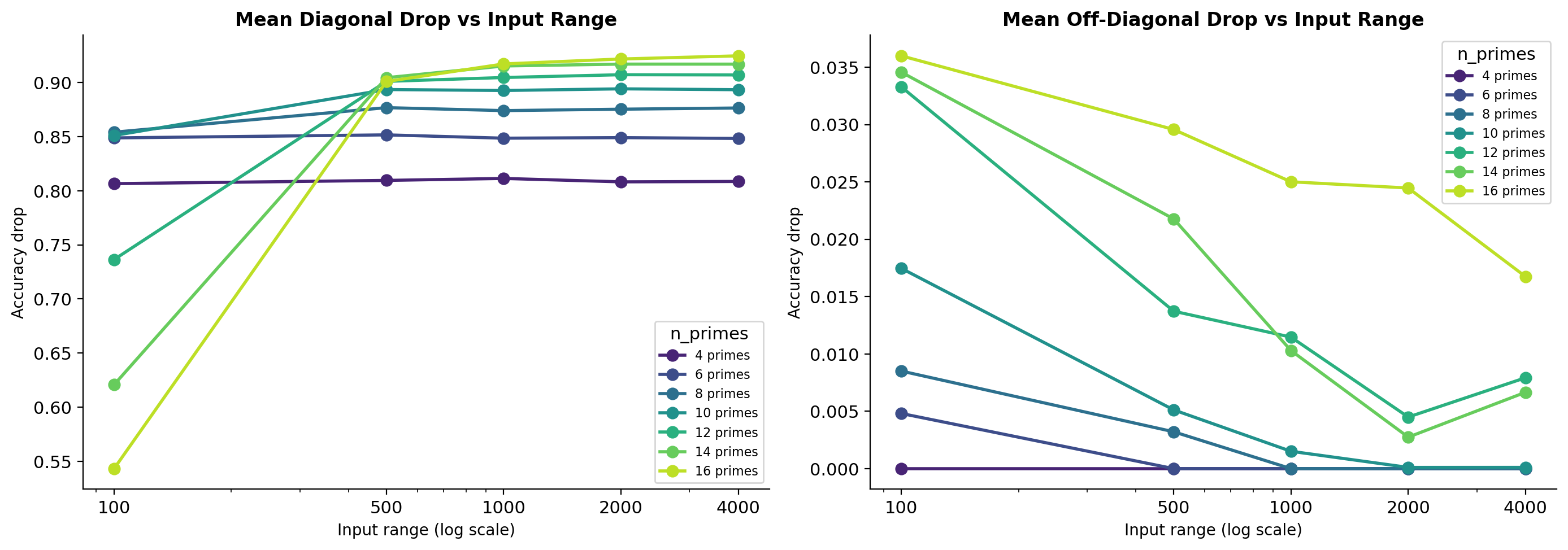}}
\caption{Mean diagonal and off-diagonal drop as a function of input 
range for each value of $|\mathcal{P}|$.}
\label{fig:app_exp1_lines}
\end{center}
\end{figure}
\clearpage
\subsection{Experiment 2: Per-Composite Profiles Across All Input Ranges}
\begin{figure}[H]
\begin{center}
\centerline{\includegraphics[width=0.9\linewidth]{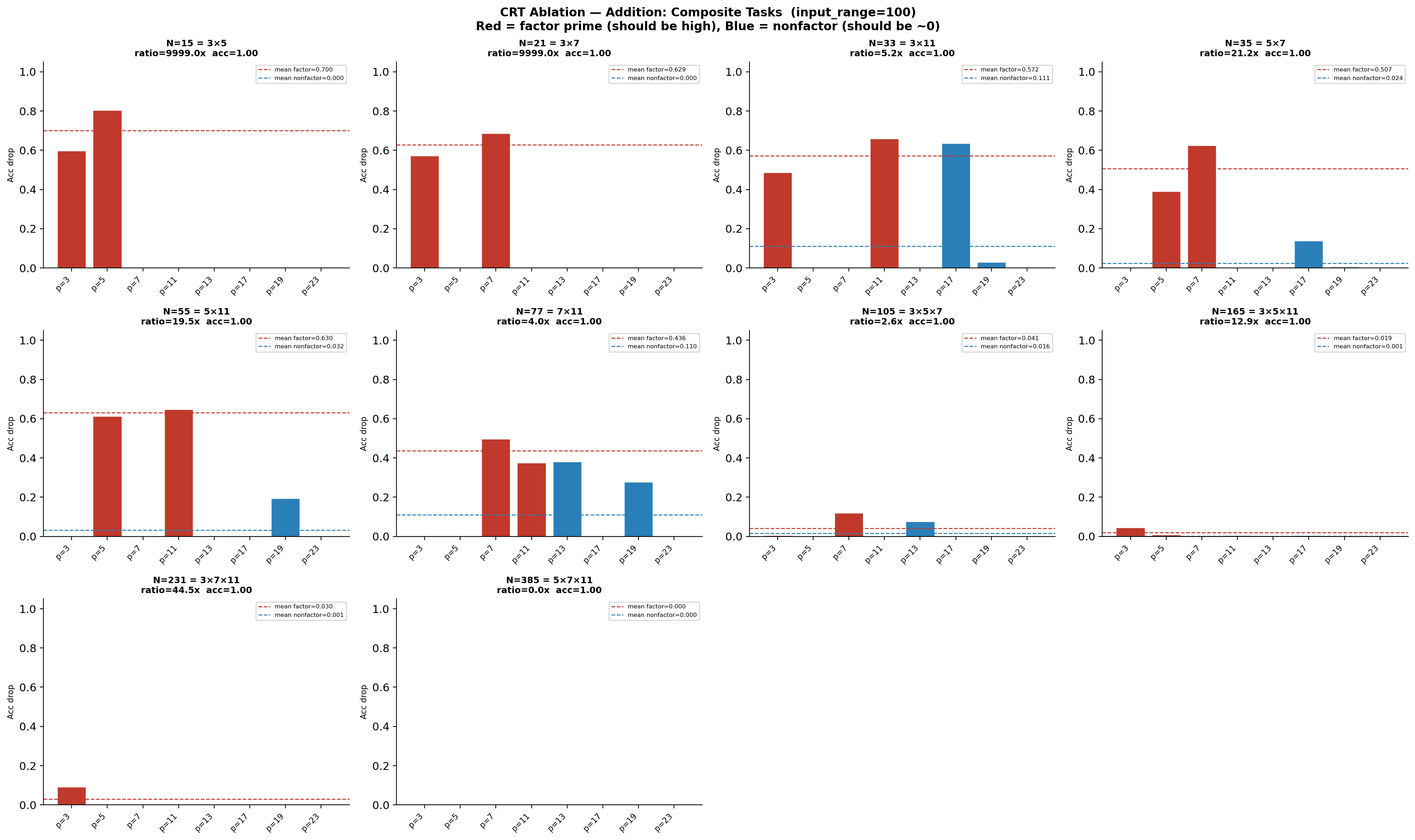}}
\caption{CRT ablation profiles, all composites, $r = 100$.}
\label{fig:app_exp2_r100}
\end{center}
\end{figure}
\begin{figure}[H]
\begin{center}
\centerline{\includegraphics[width=0.9\linewidth]{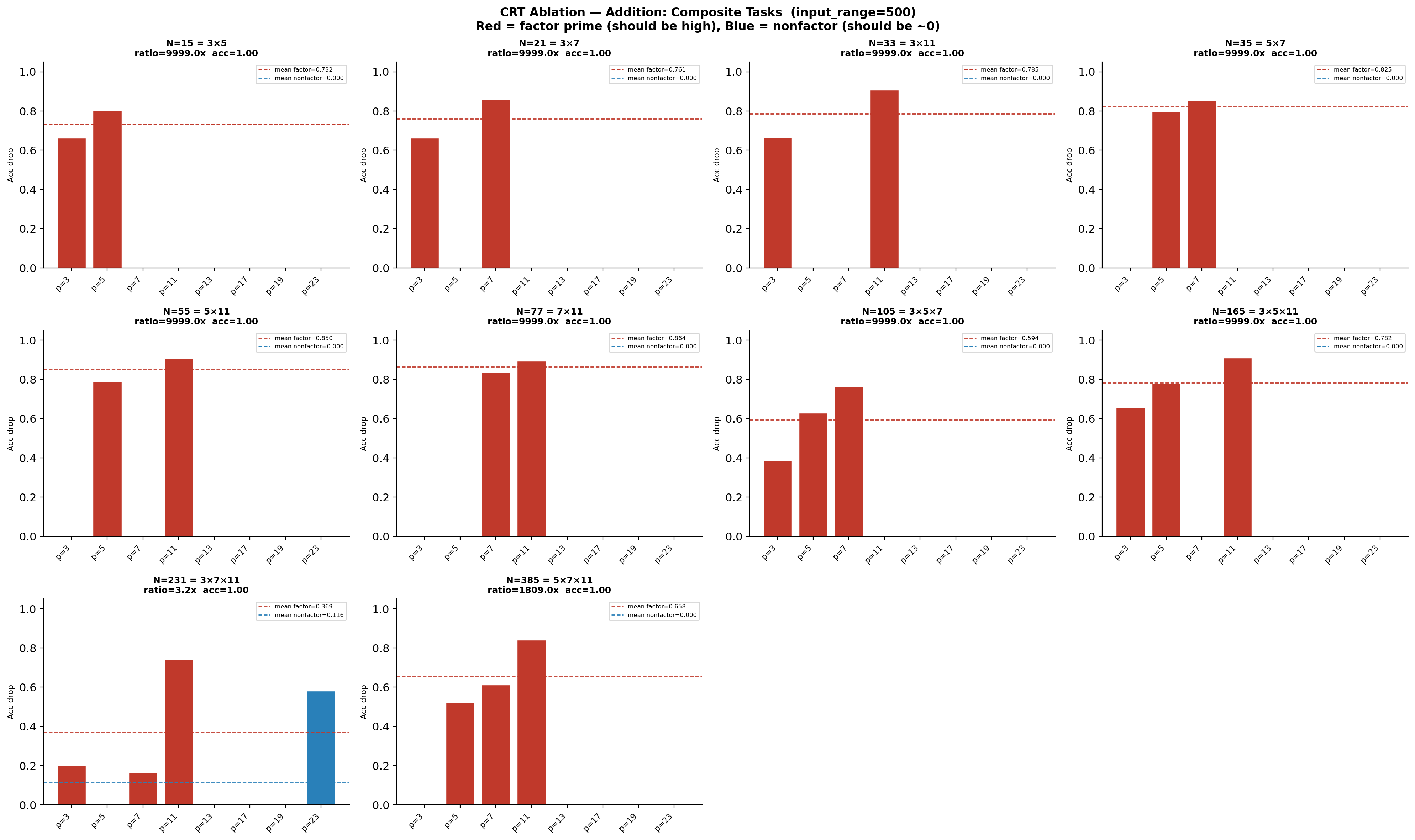}}
\caption{CRT ablation profiles, all composites, $r = 500$.}
\label{fig:app_exp2_r500}
\end{center}
\end{figure}
\begin{figure}[H]
\vskip 0.2in
\begin{center}
\centerline{\includegraphics[width=0.9\linewidth]{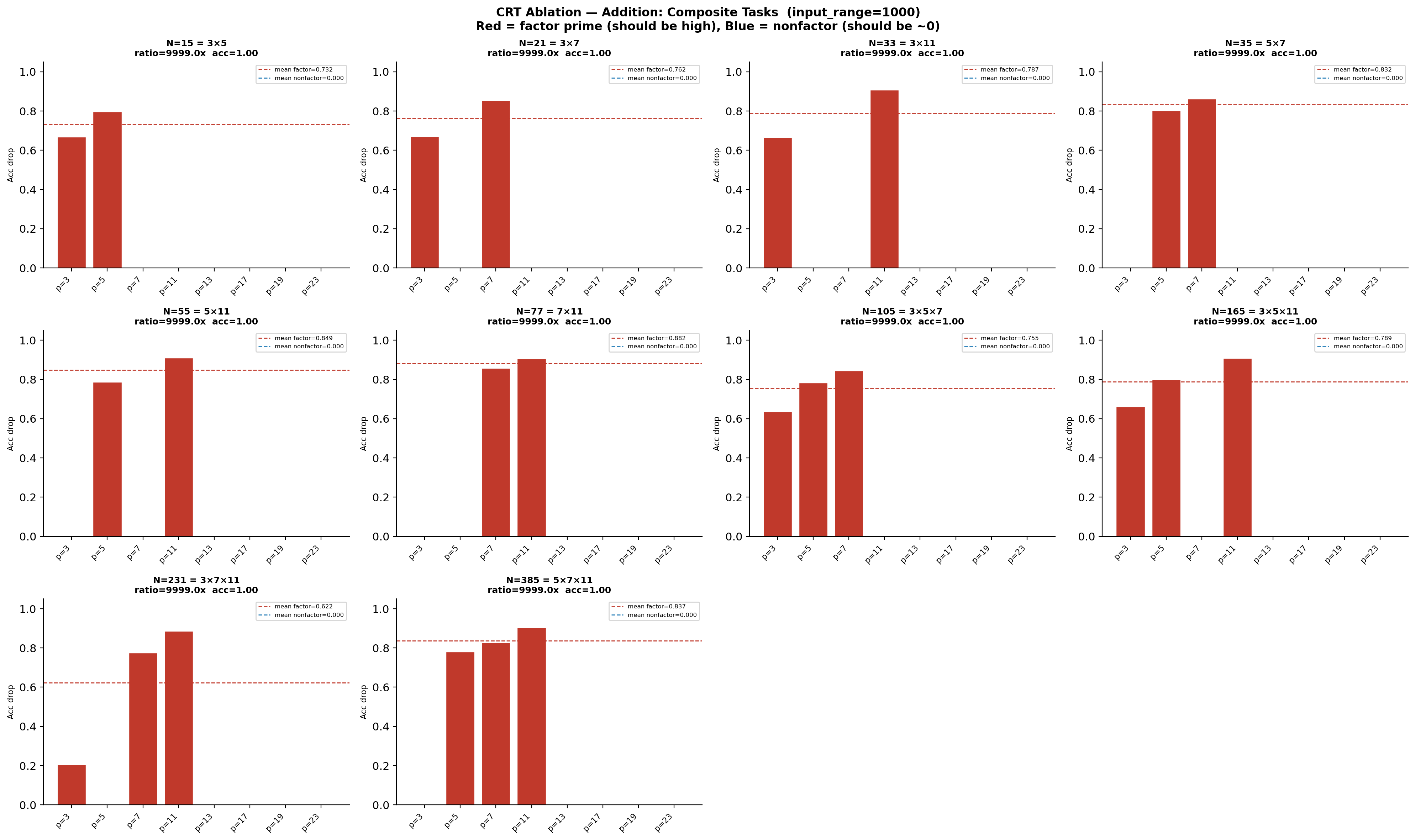}}
\caption{CRT ablation profiles, all composites, $r = 1000$.}
\label{fig:app_exp2_r1000}
\end{center}
\end{figure}
\begin{figure}[H]
\vskip 0.2in
\begin{center}
\centerline{\includegraphics[width=0.9\linewidth]{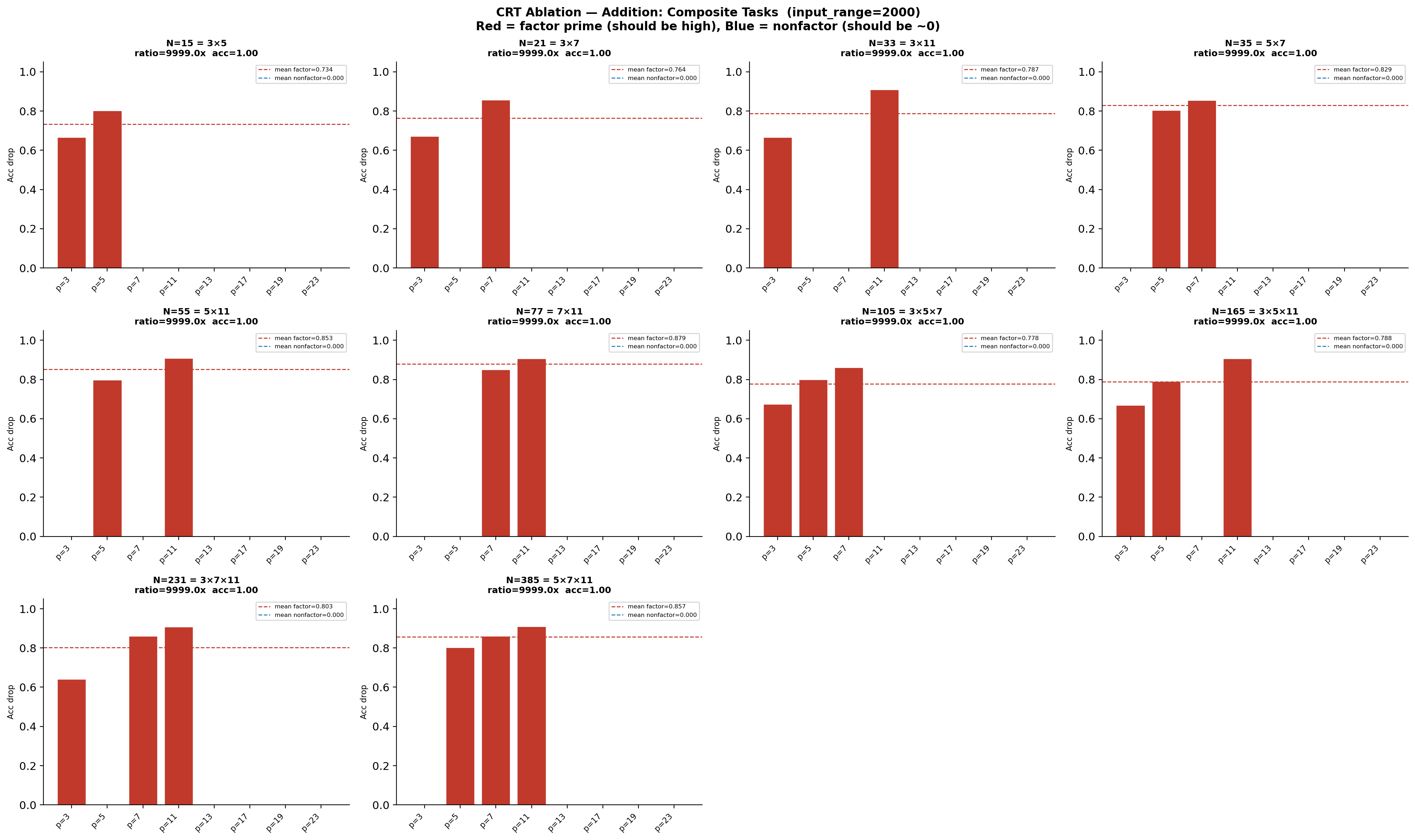}}
\caption{CRT ablation profiles, all composites, $r = 2000$.}
\label{fig:app_exp2_r2000}
\end{center}
\end{figure}
\begin{figure}[H]
\vskip 0.2in
\begin{center}
\centerline{\includegraphics[width=0.9\linewidth]{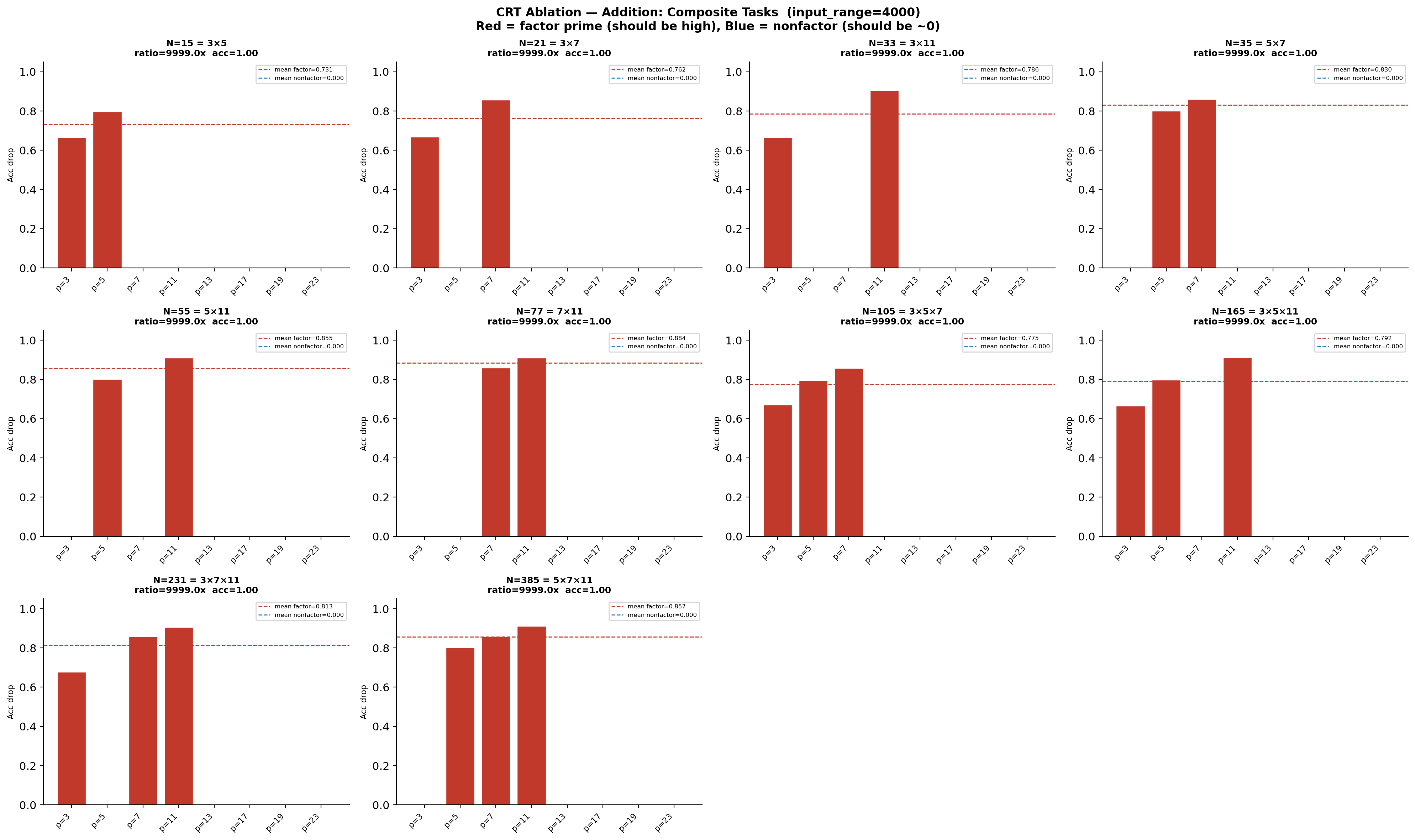}}
\caption{CRT ablation profiles, all composites, $r = 4000$. }
\label{fig:app_exp2_r4000}
\end{center}
\end{figure}

\end{document}